\documentclass[default,iicol]{sn-jnl}% Default with double column layout

\usepackage{graphicx}
\usepackage{adjustbox}
\usepackage{multirow}
\usepackage{gensymb}
\usepackage{lscape}

\usepackage{booktabs}
\usepackage{bbm}
\usepackage[load-configurations=version-1]{siunitx}
\usepackage{amsmath}

\usepackage{listings}
\usepackage{xcolor}

\definecolor{codegreen}{rgb}{0,0.6,0}
\definecolor{codegray}{rgb}{0.5,0.5,0.5}
\definecolor{codepurple}{rgb}{0.58,0,0.82}
\definecolor{backcolour}{rgb}{0.95,0.95,0.92}

\lstdefinestyle{mystyle}{
    backgroundcolor=\color{backcolour},   
    commentstyle=\color{codegreen},
    keywordstyle=\color{magenta},
    numberstyle=\tiny\color{codegray},
    stringstyle=\color{codepurple},
    basicstyle=\ttfamily\footnotesize,
    breakatwhitespace=false,         
    breaklines=true,                 
    captionpos=b,                    
    keepspaces=true,                 
    numbers=left,                    
    numbersep=5pt,                  
    showspaces=false,                
    showstringspaces=false,
    showtabs=false,                  
    tabsize=2
}

\usepackage{amsmath,amsfonts,bm}

\def\eqref#1{equation~\ref{#1}}
\def\1{\bm{1}}

\def\rvtheta{{\bm{\theta}}}

\def\vphi{{\bm{\phi}}}

\def\rvTheta{{\bm{\Theta}}}

\def\rvu{{\mathbf{i}}}

\def\rvu{{\mathbf{u}}}

\def\rvw{{\mathbf{B}}}
\def\rvx{{\mathbf{x}}}

\def\vzero{{\bm{0}}}
\def\law{{\mathrm{Law}}}
\def\vone{{\bm{1}}}

\def\vtheta{{\bm{\theta}}}

\def\vb{{\bm{u}}}

\def\vu{{\bm{u}}}
\def\vv{{\bm{v}}}

\def\vx{{\bm{x}}}
\def\vy{{\bm{y}}}
\def\vz{{\bm{z}}}

\def\mA{{\bm{A}}}

\def\mW{{\bm{W}}}
\def\mX{{\bm{X}}}

\def\mZ{{\bm{Z}}}

\def\mPhi{{\bm{\Phi}}}

\DeclareMathAlphabet{\mathsfit}{\encodingdefault}{\sfdefault}{m}{sl}
\SetMathAlphabet{\mathsfit}{bold}{\encodingdefault}{\sfdefault}{bx}{n}

\def\gF{{\mathcal{F}}}

\def\gL{{\mathcal{L}}}

\def\gN{{\mathcal{N}}}
\def\gO{{\mathcal{O}}}

\def\gU{{\mathcal{U}}}

\def\gZ{{\mathcal{Z}}}

\def\sI{{\mathbb{I}}}

\def\sP{{\mathbb{P}}}
\def\sQ{{\mathbb{Q}}}
\def\sR{{\mathbb{R}}}
\def\sS{{\mathbb{S}}}

\newcommand{\bigCI}{\mathrel{\text{\scalebox{1.07}{$\perp\mkern-10mu\perp$}}}}

\newcommand{\E}{\mathbb{E}}
\newcommand{\W}{\mathbb{W}}

\renewcommand{\R}{\mathbb{R}}

\def\calD{{\mathcal{D}}}

\let\log\relax
\DeclareMathOperator{\log}{ln}

\def\Q{{\mathbb{Q}}}

\def\I{{\mathbb{I}_d}}
\def\ind{{\bm{1}}}
\def\R{{\mathbb{R}}}

\newcommand{\KL}{D_{\mathrm{KL}}}

\def\vec{{\text{vec}}}

\DeclareMathOperator*{\argmin}{arg\,min}
\DeclareMathOperator*{\arginf}{arg\,inf}

\DeclareMathOperator*{\logsumexp}{logsumexp}

\newcommand{\psimplex}[1]{\Delta_{#1}}
\newcommand{\bY}{\bm{Y}}
\newcommand{\bM}{\bm{M}}
\newcommand{\bA}{\bm{A}}

\newcommand{\by}{\bm{y}}
\newcommand{\bmm}{\bm{m}}
\newcommand{\ba}{\bm{a}}

\newcommand{\dd}{\mathrm{d}}

\jyear{2021}%

\theoremstyle{thmstyleone}%
\newtheorem{theorem}{Theorem}%  meant for continuous numbers
\theoremstyle{thmstyleone}%
\newtheorem{remark}{Remark}%
\newtheorem{observation}{Observation}%
\newtheorem{corollary}{Corollary}

\newtheorem{lemma}{Lemma}

\theoremstyle{thmstyleone}%
\newtheorem{definition}{Definition}%

\makeatletter
\newtheorem*{rep@theorem}{\rep@title}
\newcommand{\newreptheorem}[2]{%
\newenvironment{rep#1}[1]{%
 \def\rep@title{#2 \ref{##1}}%
 \begin{rep@theorem}}%
 {\end{rep@theorem}}}
\makeatother

\newreptheorem{proposition}{Proposition}
\newreptheorem{corollary}{Corollary}
\newreptheorem{theorem}{Theorem}
\newreptheorem{lemma}{Lemma}
\newreptheorem{observation}{Observation}
\newreptheorem{remark}{Remark}

\newcommand\blfootnote[1]{%
  \begingroup
  \renewcommand\thefootnote{}\footnote{#1}%
  \addtocounter{footnote}{-1}%
  \endgroup
}

\begin{document}

\title[Bayesian Learning via N-SFS]{Bayesian Learning via Neural Schrödinger-Föllmer Flows } 


\author*[1]{\fnm{Francisco} \sur{Vargas}}\email{fav25@cam.ac.uk}

\author[2]{\fnm{Andrius} \sur{Ovsianas}}\email{ao464@cam.ac.uk}

\author[4]{\fnm{David} \sur{Fernandes}}\email{dlf28@bath.ac.uk}

\author[2]{\fnm{Mark} \sur{Girolami}}\email{mag92@cam.ac.uk}

\author[1]{\fnm{Neil D.} \sur{Lawrence}}\email{ndl21@cam.ac.uk}

\author[4]{\fnm{Nikolas} \sur{Nüsken}}\email{nuesken@uni-potsdam.de}

% : 15 JJ Thomson Ave, Cambridge CB3 0FD
\affil*[1]{\orgdiv{Department of Computer Science}, \orgname{Cambridge University}, \orgaddress{\city{Cambridge}, \postcode{CB3 0FD}, \country{UK}}}

% Engineering Dept, Trumpington St, Cambridge CB2 1PZ
\affil[2]{\orgdiv{Department of Engineering}, \orgname{Cambridge University}, \orgaddress{\city{Cambridge}, \postcode{CB2 1PZ}, \country{UK}}}

%  University Of Bath, Claverton Down, Bath BA2 7PB
\affil[3]{\orgdiv{Department of Computer Science}, \orgname{University Of Bath}, \orgaddress{\city{Bath}, \postcode{Bath BA2 7PB}, \country{UK}}}

% Haus 9, Karl-Liebknecht-Straße 24, 14476 Potsdam, Germany
\affil[4]{\orgdiv{Institute of Mathematics}, \orgname{University of Potsdam}, \orgaddress{ \city{Potsdam}, \postcode{14476}, \country{Germany}}}

% % : 15 JJ Thomson Ave, Cambridge CB3 0FD
% \affil*[1]{\orgdiv{Department of Computer Science and Technology}, \orgname{Cambridge University}, \orgaddress{\street{15 JJ Thomson Ave}, \city{Cambridge}, \postcode{CB3 0FD}, \country{UK}}}

% % Engineering Dept, Trumpington St, Cambridge CB2 1PZ
% \affil[2]{\orgdiv{Department of Engineering}, \orgname{Cambridge University}, \orgaddress{\street{Trumpington St}, \city{Cambridge}, \postcode{CB2 1PZ}, \country{UK}}}

% %  University Of Bath, Claverton Down, Bath BA2 7PB
% \affil[3]{\orgdiv{Department of Computer Science}, \orgname{University Of Bath}, \orgaddress{\street{Claverton Down}, \city{Bath}, \postcode{Bath BA2 7PB}, \country{UK}}}

% % Haus 9, Karl-Liebknecht-Straße 24, 14476 Potsdam, Germany
% \affil[3]{\orgdiv{Institute of Mathematics}, \orgname{University of Potsdam}, \orgaddress{\street{Karl-Liebknecht-Straße 24}, \city{Potsdam}, \postcode{14476}, \country{Germany}}}


% \begin{abstract}

% \end{abstract}

\abstract{In this work we explore a new framework for approximate Bayesian inference in large datasets based on stochastic control. We advocate stochastic control as a finite time and low variance alternative to popular steady-state methods such as stochastic gradient Langevin dynamics (SGLD). Furthermore, we discuss and adapt the existing theoretical guarantees of this framework and establish connections to already existing VI routines in SDE-based models.}


\maketitle


% \twocolumn[

% \aistatstitle{Bayesian Learning via Neural Schrödinger-Föllmer Flows}

% \aistatsauthor{ Author 1 \And Author 2 \And  Author 3 }

% \aistatsaddress{ Institution 1 \And  Institution 2 \And Institution 3 } ]

% \begin{keywords}%
%   Schrödinger Bridge Problem, Föllmer Drift, Stochastic Control, Bayesian Inference, Bayesian Deep Learning.
% \end{keywords}

\section{Introduction}\label{sec:intro}

Steering\blfootnote{Published at Statistics and Computing, 2022.} a stochastic flow from one distribution to another across the space of probability measures is a well-studied problem initially proposed in \citet{schrodinger1932theorie}. There has been recent interest in the machine learning community in these methods for generative modelling, sampling, dataset imputation and optimal transport \citep{wang2021deep,de2021diffusion,huang2021schrodinger,bernton2019schr,vargasshro2021, chen2022likelihood,cuturi2013sinkhorn,maoutsa2021deterministic,reich2019data}.

We consider a particular instance of the Schrödinger bridge problem (SBP), known as the Schrödinger-Föllmer process (SFP). In machine learning, this process has been proposed for sampling and generative modelling \citep{huang2021schrodinger,tzen2019theoretical} and in molecular dynamics for rare event simulation and importance sampling \citep{hartmann2012efficient,hartmann2017variational}; here we apply it to Bayesian inference. We show that a control-based formulation of the SFP has deep-rooted connections to variational inference and is particularly well suited to Bayesian inference in high dimensions. This capability arises from the SFP's characterisation as an optimisation problem and its parametrisation through neural networks \citep{tzen2019theoretical}. Finally, due to the variational characterisation that these methods possess, many low-variance estimators \citep{richter2020vargrad, nusken2021solving,roeder2017sticking,xu2021infinitely} are applicable to the SFP formulation we consider. 

We reformulate the Bayesian inference problem by constructing a stochastic process $\rvTheta_t$ which at a fixed time $t=1$ will generate samples from a pre-specified posterior $p(\vtheta  \vert   \mX)$, i.e. $\law \rvTheta_1 =p(\vtheta  \vert   \mX) $, with dataset $\mX=\{\rvx_i\}_{i=1}^N$, and where the model is given by:
\begin{align}
 \rvtheta &\sim p(\rvtheta),\nonumber \\
     \rvx_i \vert   \rvtheta &\sim p(\vx_i \vert   \rvtheta). \quad \mathrm{iid}\label{eq:bayes_inf}
 \end{align}
Here the prior $p(\rvtheta)$ and the likelihood $p(\rvx_i \vert   \rvtheta)$ are user-specified. Our target is $\pi_1(\rvtheta) = \frac{p(\mX\vert   \rvtheta)p(\rvtheta)}{\gZ}$,
where $\gZ = \int \prod_i p(\vx_i \vert   \rvtheta)p(\rvtheta) d\vtheta$.
This formulation is reminiscent of the setup proposed in the previous works \citep{grenander1994representations,roberts1996exponential, girolami2011riemann, welling2011bayesian} and covers many Bayesian machine-learning models, but our formulation has an important difference. SGLD relies on a diffusion that reaches the posterior as its equilibrium state when time approaches infinity. In contrast, our dynamics are \emph{controlled} and the posterior is reached in finite time (bounded time). The benefit of this property is elegantly illustrated in Section 3.2 of \cite{huang2021schrodinger} where they rigorously demonstrate that even under an Euler approximation the proposed approach reaches a Gaussian target at time $t\!=\!1$ whilst SGLD does not.\\
 
\noindent \textbf{Contributions:} The main contributions of this work can be detailed as follows:
 \begin{itemize}
 \itemsep0em 
     \item In this work we scale and apply the theoretical framework proposed in \citep{dai1991stochastic,tzen2019neural} to sample from posteriors in large scale Bayesian machine learning tasks such as Bayesian Deep learning.  We study the robustness of the predictions under this framework as well as evaluate their uncertainty quantification.
     \item More precisely we propose an amortised parametrisation that allows scaling  models with local and global variables to large datasets.
     \item We explore and provide further theoretical backing (Section \ref{subsec:stl}) to the ``sticking the landing'' estimator provided by \cite{xu2021infinitely}.
     \item Overall we empirically demonstrate that the stochastic control framework offers a promising direction in Bayesian machine learning, striking the balance between theoretical/asymptotic guarantees found in MCMC methods \citep{hastings1970monte,duane1987hybrid,neal2011mcmc,brooks2011handbook} and more practical approaches such as variational inference \citep{blei2003latent}.
 \end{itemize}

\begin{figure*}[t!]
    \centering
    \begin{minipage}{0.3\linewidth }
    \centering
    \includegraphics[width=\linewidth ]{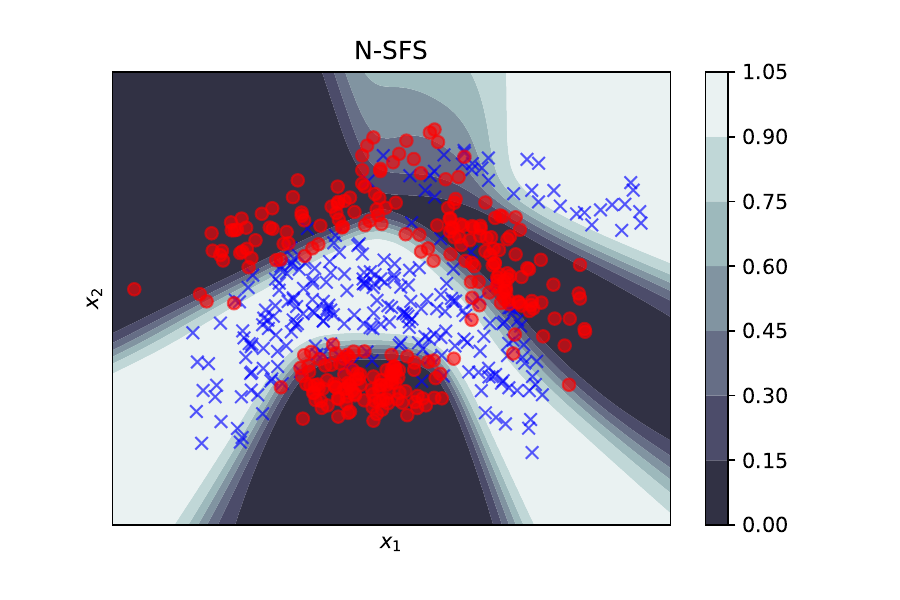}
    \end{minipage}
    \hspace*{\fill}\begin{minipage}{0.3\linewidth }
    \centering
    \includegraphics[width=\linewidth ]{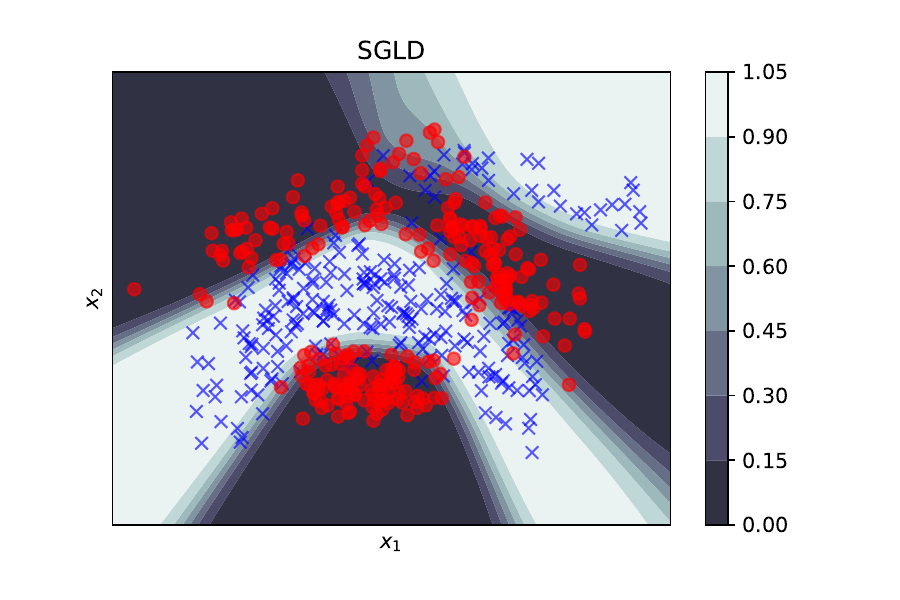}
    \end{minipage}
    \hspace*{\fill}\begin{minipage}{0.3\linewidth }
    \centering
    \includegraphics[width=\linewidth ]{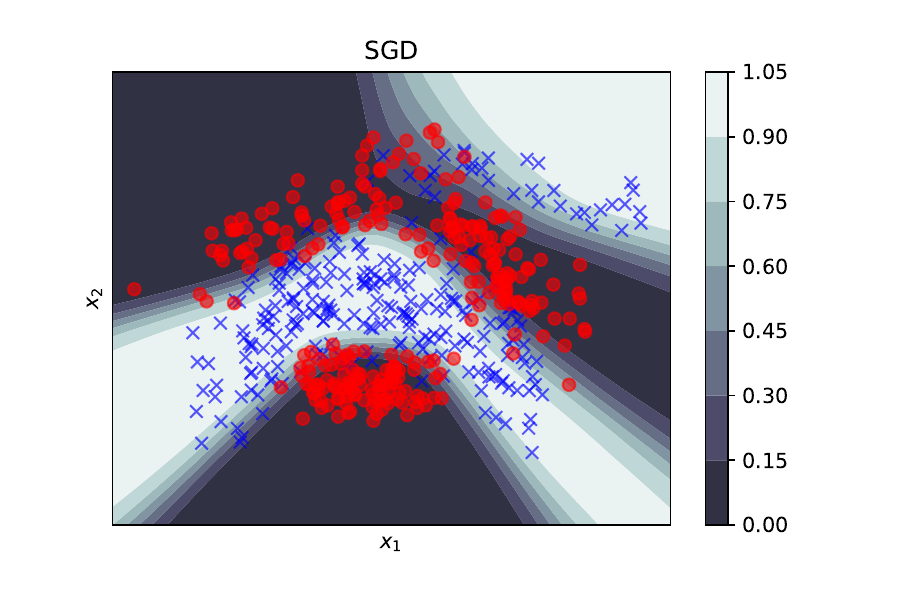}
    \end{minipage}
    \caption{Predictive posterior contour plots on the banana dataset \citep{diethe13benchmark}. Test accuracies:  $0.8928 \pm0.0056, 0.8913\pm 0.0105, 0.8800\pm0.0063$ and test ECEs: $ 0.0229\pm0.0062 ,  0.0253   \pm 0.0042, 0.0267 \pm 0.0083$ for N-SFS, SGLD, and SGD respectively. We observe that N-SFS obtains the highest test accuracy whilst preserving the lowest ECE.} \label{fig:banana}
\end{figure*}

\subsection{Notation}

Throughout the paper we consider path measures (denoted as $\Q$ or $\sS$) on the 
space of continuous 
functions $\Omega = C([0, 1], \sR^d)$. Random processes associated with such path
measures $\Q$ are denoted as $\rvTheta$ and their time-marginal distributions as
$\Q_t = (\rvTheta_t)_\#\Q$ (which are just pushforward measures). Given two
marginal distributions $\pi_0$ and $\pi_1$ we write $\calD(\pi_0, \pi_1) = \{\Q : \Q_0 = \pi_0, \Q_1 = \pi_1 \}$
for the set of all path measures with given marginal distributions at the initial
and final times. We denote by $\Q^{\rvu, \pi}$ the path measure of
the following Stochastic Differential Equation (SDE):
\begin{equation}\label{eq:sde1}
    d \rvTheta_t = \rvu(t, \rvTheta_t)d t + \sqrt{\gamma} d \rvw_t, \;\;\; \rvTheta_0 \sim \pi
\end{equation}\vspace{-2em}

\noindent
(we drop the dependence on $\gamma$ since it is fixed) and we write 
$\W^\gamma = \Q^{0, \delta_0}$ for the Wiener measure. We will write 
$\frac{d\Q}{d\sS}$ for the Radon-Nikodym derivative (RND) of $\Q$ w.r.t. $\sS$.

\subsection{Schrödinger-Föllmer Processes}

\begin{definition}
    (Schrödinger-Bridge Process) Given a reference process $\sS$ %$\Q^{\gamma, \vb, \pi}$
    and two measures $\pi_0$ and $\pi_1$ the Schrödinger bridge distribution is 
    given by

    \begin{equation}
        \Q^* = \arginf_{\Q \in \calD(\pi_0, \pi_1)} \KL\left(\Q \big\vert  \big\vert   \sS \right),
    \end{equation}
    where $\sS$ acts as a ``prior''. 

    It is known \citep{leonard2013survey} that if $\sS = \Q^{\rvu, \pi}$, 
    $\Q^*$ is induced by an SDE with a modified drift:

    \begin{equation}
            \dd\rvTheta_t = \rvu^*(t, \rvTheta_t)\dd t +  \sqrt{\gamma} \dd \rvw_t, \;\;\; \rvTheta_0 \sim \pi_0,
    \end{equation}

    i.e. $\Q^* = \Q^{\rvu^*, \pi_0}$. Solution of this SDE is called the 
    Schr{\"o}dinger-Bridge Process (SBP). 
\end{definition}

\begin{definition}
    (Schrödinger-Föllmer Process) The SFP is an SBP where $\pi_0=\delta_0$ and
    the reference process $\sS = \W^\gamma$ is the Wiener measure.
\end{definition}
 
The SFP differs from the general SBP in that, rather than constraining the initial 
distribution to $\delta_0$, the SBP considers \emph{any} initial distribution 
$\pi_0$. The SBP also involves general Itô SDEs associated with $\Q^{\vb, \pi}$ 
as the dynamical prior, compared to the SFP which restricts attention to Wiener 
processes as priors.
 
The advantage of considering this more limited version of the SBP is that it admits 
a closed-form characterisation of the solution to the Schrödinger system 
\citep{leonard2013survey,wang2021deep,pavon2018data} which allows for an 
unconstrained formulation of the problem. For accessible introductions to the SBP 
we suggest \citep{pavon2018data,vargasshro2021}. Now we will consider instances 
of the SBP and the SFP where $\pi_1=p(\vtheta  \vert   \mX)$.

\subsubsection{Analytic Solutions and the Heat Semigroup}

Prior work \citep{pavon1989stochastic,dai1991stochastic,tzen2019theoretical,huang2021schrodinger} has explored the properties of SFPs via a closed form formulation of the Föllmer drift expressed in terms of expectations over Gaussian random variables known as the heat semigroup. The seminal works  \citep{pavon1989stochastic,dai1991stochastic,tzen2019theoretical} highlight how this formulation of the Föllmer drift characterises an exact sampling scheme for a target distribution and how it could potentially be used in practice. The recent work by \cite{huang2021schrodinger} builds on \cite{tzen2019theoretical} and explores estimating the optimal drift in practice via the heat semigroup formulation using a Monte Carlo approximation. Our work aims to take the next step and scale the estimation of the Föllmer drift to high dimensional cases \citep{graves2011practical, hoffman2013stochastic}. In order to do this we must move away from the heat semigroup and instead consider the dual formulation of the Föllmer drift in terms of a stochastic control problem \citep{tzen2019theoretical}.

In the setting when $\pi_0=\delta_0$  we can express the optimal SBP drift as follows
\begin{align}
    \label{eq:optimal drift}
    \rvu^*(t, \vx) = \nabla_\vx\ln\E_{\rvTheta \sim \sS}\left[\frac{d\pi_1}{d \sS_1}(\rvTheta_1)\Big\vert   \rvTheta_t =  \vx\right]
\end{align}

\begin{definition}
    The Euclidean heat semigroup $Q_t^\gamma, \; t\geq 0$, acts on bounded measurable functions $f: \sR^d \rightarrow \sR$ as $Q^\gamma_t f(\vx) = \int_{\sR^d} f\left(\vx+\sqrt{t}\vz\right) \gN(\vz \vert  \vzero, \gamma\sI) d\vz  =\E_{\vz \sim \gN(\vzero, \gamma\sI)}\left[f\left(\vx+\sqrt{t}\vz\right)\right].$
\end{definition}
In the SFP case where $\sS=\W^\gamma$, the optimal drift from 
Equation \ref{eq:optimal drift} can be written in terms of the heat 
semigroup, $\rvu^*(t, \vx) = \nabla_\vx\ln Q^\gamma_{1-t} \left[\frac{d \pi_1}{d\gN(\vzero, \gamma\sI)}(\vx)\right]$.
Note that an SDE with the heat semigroup induced drift

\begin{equation}
    d\rvTheta_t = \nabla_{\rvTheta_t}\ln Q^\gamma_{1-t} \left[\frac{d\pi_1}{d\gN(\vzero, \gamma\sI)}(\rvTheta_t )\right] d t + \sqrt{\gamma} d\rvw_t
\end{equation}
satisfies $ \law \rvTheta_1 = \pi_1$, that is, at $t=1$ these processes are distributed 
according to our target distribution of interest $\pi_1$.

\subsubsection{Schrödinger-Föllmer Samplers}

\cite{huang2021schrodinger} carried out preliminary work on empirically exploring the success of using the heat semigroup formulation of SFPs in combination with the Euler-Mayurama (EM) discretisation to sample from target distributions in a method they call Schrödinger-Föllmer samplers (SFS). More precisely the SFS approach proposes estimating the Föllmer drift via:

\begin{align} \label{eq:sfs_est}
     \hat{\rvu}^*(t, \vx) = \frac{\frac{1}{S}\sum_{s=1}^S \vz_s f(\vx + \sqrt{1-t} \vz_s)}{\frac{\sqrt{1-t}}{S}\sum_{s=1}^S f(\vx + \sqrt{1-t} \vz_s)},
\end{align}
where $\vz_s \sim \gN(\vzero, \gamma \sI)$ and $f=\frac{d\pi_1}{d\gN(\vzero, \gamma\sI)}$. Whilst this estimator enjoys sound theoretical properties \citep{huang2021schrodinger} it falls short in practice for the following reasons:
\begin{figure*}
    \centering
    \includegraphics[width=\textwidth]{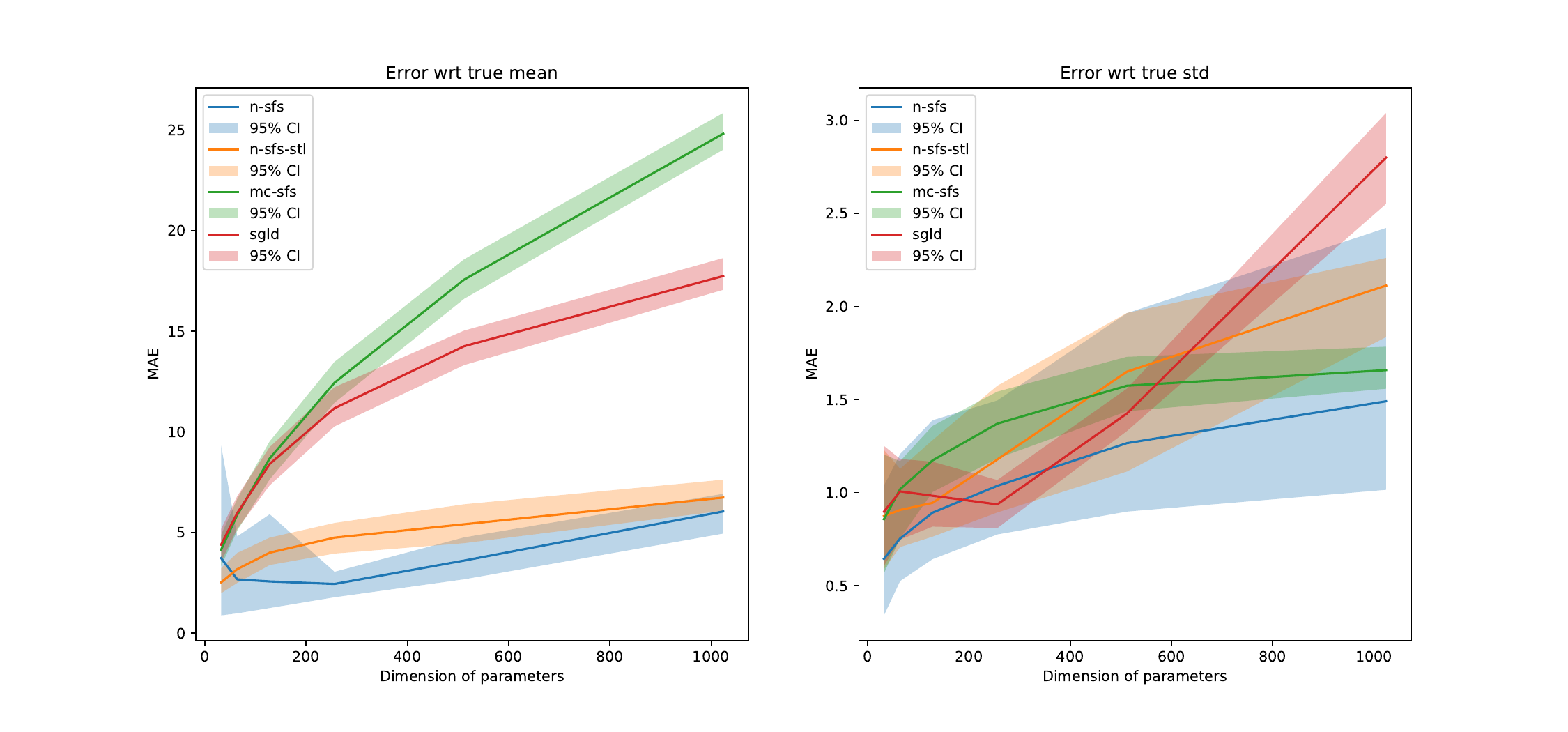}
    \caption{Comparison between MC-SFS and N-SFS under similar computational constraints. Target distribution is the Gaussian posterior induced by a Bayesian linear regression model, we plot the error of the first and second posterior predictive moments between the true posterior predictive and the listed approximations. We found increasing the number of steps in SGLD drove the errors closer to 0 however when increasing the dimensions this threshold also increased notably. This illustrates the advantages of having a target at a finite time rather than at equilibrium.}
    \label{fig:mc-sfs}
\end{figure*}
\begin{itemize}
    \item The term $f$ involves the product of PDFs evaluated at samples rather than a log product and is thus often very unstable numerically. In Appendix \ref{apdx:sfs} we provide a more stable implementation of Equation \ref{eq:sfs_est} exploiting the logsumexp trick and properties of the Lebesgue integral. 
    \item In it's current form the estimator does not admit low variance estimators (e.g. Variatonal Inference), being a Monte Carlo estimator it is prone to high variance.
    \item Both empirically and theoretically we found the computational running time of the above approach to be considerably slower than the other methods we compare to. At test time SFS has a computational complexity of $\gO(T S\#_f(d) )$ where $T=\Delta t^{-1}$ , $S$ is the number of Monte Carlo samples and $\#_f(d)$ is the cost of evaluating the RND $f$ which at best is linear in $d$. Meanwhile our proposed approach enjoys a cost of $\gO(T \#_{\vu_\phi}(d) )$ where $\#_{\vu_\phi}(d)$ is the forward pass through a neural network approximating the Föllmer drift.
\end{itemize}

In practice we found this implementation to be too numerically unstable and unable to produce reasonable results even in low dimensional examples in order to carry out a fair comparison we reformulated Equation \ref{eq:sfs_est} stably, the stable formulation and its derivation can be found in Appendix \ref{apdx:sfs}.

In this work build on \cite{huang2021schrodinger} by considering a formulation of the Schrödinger-Föllmer process that is suitable for the high dimensional settings arising in Bayesian ML. Our work will focus on a dual formulation of the optimal drift that is closer to variational inference and thus admits the scalable and flexible parametrisations used in ML.
\begin{algorithm*}
\caption{Optimization of N-SFS with Stochastic Mini-batches.}\label{alg:svi}
\begin{algorithmic}[1]
\State {\textbf{Input: }data set $\mX\!=\!\{\rvx_i\}_{i=1}^N$, initialized drift NN $\vu_\phi$, parameter dimension $d$, \# of iterations $M$, batch size $B$, \# of EM discretization steps $k$, \# of MC samples $S$ , diffusion coefficient $\gamma$.}
\State{\textbf{Initialise: }$\Delta t \gets \frac{1}{k}$,  $\;\;t_j \gets j\Delta t$ for all $j = 0, \dots, k$}

\For{$i = 1, \dots, M$}
    \State{Initialize $\rvTheta_0^{s} \gets 0 \in \R^d$ for all $s = 1, \dots, S$}
    \State{$\{\Theta_j^{s\phi}\}_{j=1}^k \gets \text{Euler-Maruyama}(\vu_\phi, \rvTheta_0^{s}, \Delta t)$ for all $s = 1, \dots, S$}
    \State{Sample $\rvx_{r_1}, \dots, \rvx_{r_B} \sim \mX$}
    \State{ $g \!\gets \nabla_\phi \left(\frac{1}{S}\sum\limits_{s=1}^S\sum\limits_{j=0}^k \left(\!\!\vert  \vert  \vu_\phi(\rvTheta_j^{s\phi}, t_j)\vert  \vert  ^2\Delta t \!- \ln\!\left(\!\frac{p(\rvTheta_k^{s\phi} )}{\gN(\rvTheta_k^{s\phi}\vert  \bm{0}, \gamma \I)}\!\!\right) \!+\! \frac{N}{B}\sum\limits_{j=1}^B \ln p(\rvx_{r_j}\vert  \rvTheta_k^{s\phi})\right)\!\!\right)$ }
    \State{$\phi \gets \text{Gradient Step}(\phi, g)$}
\EndFor
\State{\textbf{Return: } $\vu_\phi$}
\end{algorithmic}
\end{algorithm*}

\section{Stochastic Control Formulation} \label{sec:csfp}

In this section, we introduce a particular formulation of the Schrödinger-Föllmer process in the context of the Bayesian inference problem in Equation \ref{eq:bayes_inf}. In its most general setting of sampling from a target distribution, this formulation was known to \cite{dai1991stochastic}. \cite{tzen2019theoretical} study the theoretical properties of this approach in the context of generative models \citep{kingma2021variational,goodfellow2014generative}, finally \cite{opper2019variational} applies this formulation to time series modelling. In contrast our focus is on the estimation of a Bayesian posterior for a broader class of models than \citeauthor{tzen2019theoretical} explore.

\begin{corollary}\label{col:main}
    Define

    \begin{equation*}
        \gF_{\mathrm{DET}}(\rvu, \rvtheta) = \frac{1}{2\gamma}\int_0^1\|\rvu(t, \rvtheta_t)\|  ^2 dt - \ln\frac{ p(\mX \vert   \rvtheta_1)p(\rvtheta_1)}{\gN(\rvtheta_1\vert  \bm{0}, \gamma \I)}
    \end{equation*}

    \begin{equation*}
        J(\rvu) = \E_{\rvTheta \sim \Q^{\rvu, \delta_0}}\left[\gF_{\mathrm{DET}}(\rvu, \rvTheta)\right]
    \end{equation*}

    Then the minimiser (with $\gU$ being the set of admissible controls\footnote{Under appropriate conditions on the model in Equation \ref{eq:bayes_inf}, $\mathcal{U}$ can be taken to be the set of $C^1$-vector fields with linear growth in space, see \cite{nusken2021solving}.})
    \begin{equation}
        \rvu^{*}\!=\!\argmin_{\rvu \in \gU} J(\rvu) \label{eq:CSFP}
    \end{equation}
    satisfies $\Q_1^{\gamma, \rvu^*, \delta_0} = \frac{p(\mX\vert   \rvtheta)p(\rvtheta)}{\gZ}d\rvtheta$.

  Moreover, $\rvu^{*}$ solves the SFP with $\pi_1=p(\rvtheta\vert   \mX)$.

\end{corollary}
The objective in Equation \ref{eq:CSFP} can be estimated using an SDE discretisation, such as the EM method. Since the drift $\rvu^{*}$ is Markov, it can be parametrised by a flexible function estimator such as a neural network, as  in \cite{tzen2019theoretical}. In addition, unbiased estimators for the gradient of objective in~\eqref{eq:CSFP} can be formed by subsampling the data. In this work we will refer to the above formulation of the SFP as the Neural Schrödinger-Föllmer sampler (N-SFS) when we parametrise the drift with a neural network and  implement unbiased mini-batched estimators for this objective (Appendix \ref{apdx:svi}). This formulation of SFPs has been previously studied in the context of generative modelling / marginal likelihood estimation \citep{tzen2019theoretical}, while we focus on Bayesian inference. 

We note that recent concurrent work \citep{zhang2022path} \footnote{This work was made public on arxiv within a month of our arxiv pre-print release.} proposes an algorithm akin to ours based on \cite{dai1991stochastic,tzen2019theoretical}, however their focus is on estimating the normalising constant of unnormalised densities, while ours  is on Bayesian ML tasks such as Bayesian regression, classification and LVMs, thus our work leads to different insights and algorithmic motivations.

\subsection{Theoretical Guarantees for Neural SFS}

While the focus in \cite{tzen2019theoretical} is in providing guarantees for generative models of the form $ \vx \sim  q_\phi(\vx \vert   \mZ_1)\;,d\mZ_t = \vb_{\phi}(\mZ_t,t)dt + \sqrt{\gamma} d\rvw_t, \;\mZ_0 =\! 0,$
their results extend to our setting as they explore approximating the Föllmer drift for a generic target $\pi_1$. 

Theorem 4 in \citeauthor{tzen2019theoretical} (restated as Theorem \ref{thrm:tzen} in Appendix \ref{app:em}) motivates using neural networks to parametrise the drift in Equation \ref{eq:CSFP} as it provides a guarantee regarding the expressivity of a network parametrised drift via an upper bound on the target distribution error in terms of the size of the network.

We will now proceed to highlight how this error is affected by the EM discretisation:
\begin{corollary}\label{col:euler}
Given the network $\vv$ from Theorem \ref{thrm:tzen} it follows that the 
Euler-Maruyama discretisation of~\eqref{eq:sde1} with $\rvu=\vv$ 
induces an approximate target $\hat{\pi}^{\vv}_1$ that satisfies
\begin{align}
     \KL(\pi_1\vert  \vert   \hat{\pi}^{\vv}_1 ) \leq\left (\epsilon^{1/2} + \mathcal{O}(\sqrt{\Delta t}) \right)^2.
\end{align}
\end{corollary}
This result provides a bound of the error in terms of the depth $\Delta t^{-1}$ of the stochastic flow \citep{chen2022likelihood,zhang2021diffusion} and the size of the network that we parametrise the drift with. Under the view that NN parametrised SDEs can be interpreted as ResNets \citep{li2020scalable} we find that this result illustrates that increasing the ResNets' depth will lead to more accurate results.

\subsection{Sticking the Landing and Low Variance Estimators} \label{subsec:stl}

As with VI \citep{richter2020vargrad,roeder2017sticking}, the gradient of the objective in this study admits several low variance estimators \citep{nusken2021solving,xu2021infinitely}. In this section we formally recap what it means for an estimator to ``stick the landing'' and we prove that the estimator proposed in \citeauthor{xu2021infinitely} satisfies said property. 

The full objective being minimised in our approach is (where expectations are taken over $\rvTheta \sim \Q^{\rvu, \delta_0}$):
\begin{align}
  J(\rvu) =\E[&\gF_{\mathrm{DET}}(\rvu, \rvTheta)]\! \nonumber\\
          =\E[&\gF(\rvu, \rvTheta)]\! \nonumber\\
  =\!\E\!\Bigg[&\!\frac{1}{2\gamma}\!\!\int_0^1\!\!\!\!\vert  \vert  \rvu_t(\rvTheta_t)\vert  \vert  ^2 \!dt + \frac{1}{\sqrt{\gamma} }\int_0^1\!\!\!\!\rvu_t(\rvTheta_t)^\top  \!d\rvw_t  \nonumber \\
  & \!-\! \ln\!\Big(\!\frac{ p(\mX \vert   \rvTheta_1)p(\rvTheta_1)}{\gN(\rvTheta_1\vert  \bm{0}, \gamma \I)}\!\!\Big)\!\Bigg], 
\end{align}

noticing that in previous formulations we have omitted the It{\^o} integral as it has zero expectation (but the integral appears naturally through Girsanov's theorem). We call the estimator calculated by taking gradients of the above objective the relative-entropy estimator. The estimator proposed in \cite{xu2021infinitely} (Sticking the landing estimator) is given by:
\begin{align}
  J_{\mathrm{STL}}(\rvu) =\E[&\gF_{\mathrm{STL}}(\rvu, \rvTheta)]\! \nonumber\\
  =\!\E\!\Bigg[&\!\frac{1}{2\gamma}\!\!\int_0^1\!\!\!\!\vert  \vert  \rvu_t(\rvTheta_t)\vert  \vert  ^2 \!dt \!+ \!\frac{1}{\sqrt{\gamma} }\!\!\int_0^1\!\!\!\!\!\rvu^{\perp}_t(\rvTheta_t)^\top  \!\!d\rvw_t\nonumber \\ &\!-\! \ln\!\Big(\!\frac{ p(\mX \vert   \rvTheta_1)p(\rvTheta_1)}{\gN(\rvTheta_1\vert  \bm{0}, \gamma \I)}\!\!\Big)\!\Bigg], 
\end{align}
where $\perp$ means that the gradient is stopped/detached as in \cite{xu2021infinitely,roeder2017sticking}.

We study perturbations of $\gF$ around $\rvu^*$ by considering 
$\rvu^* + \varepsilon \vphi$, with $\vphi$ arbitrary, and $\varepsilon$ small. 
More precisely, we set out to compute (where dependence on $\rvtheta$ is dropped):
\begin{align}
\frac{\mathrm{d}}{\mathrm{d}\varepsilon}  \mathcal{F}(\rvu^* + \varepsilon \vphi)\Big\vert  _{\varepsilon = 0},
\end{align}
through which we define the definition of ``sticking the landing'':
\begin{definition}
\label{def:STL}
    We say that an estimator ``sticks the landing'' when
\begin{align}
    \frac{\mathrm{d}}{\mathrm{d}\varepsilon}  \mathcal{F}(\rvu^* + \varepsilon \vphi) \Big\vert  _{\varepsilon = 0}=0,
\end{align}
almost surely, for all smooth and bounded perturbations $\vphi$.
\end{definition}
Notice that by construction, $\vu^*$ is a global minimiser of $J$, and hence all directional derivatives vanish, 
\begin{equation}
\label{eq:zero exp}
    \frac{\mathrm{d}}{\mathrm{d}\varepsilon}  J(\rvu^* + \varepsilon \vphi) \Big\vert  _{\varepsilon = 0} = \frac{\mathrm{d}}{\mathrm{d}\varepsilon}  \mathbb{E}[\mathcal{F}(\rvu^* + \varepsilon \vphi, \rvTheta)] \Big\vert  _{\varepsilon = 0} =0.
\end{equation}
Definition \ref{def:STL} additionally demands  that this quantity is zero almost surely, and not just on average. Consequently, ``sticking the landing''-estimators will have zero-variance at $\rvu^*$.

\begin{remark}
The relative-entropy stochastic control estimator does not stick the landing.
\end{remark}
\begin{proof}
See \cite{nusken2021solving}, Theorem 5.3.1, clause 3, Equation 133 clearly indicates $\frac{\mathrm{d}}{\mathrm{d}\varepsilon}  \mathcal{F}(\rvu^* + \varepsilon \vphi) \Big\vert  _{\varepsilon = 0}\neq0$.
\end{proof}

We can now go ahead and prove that the estimator proposed by \cite{xu2021infinitely} does indeed stick the landing.

\begin{theorem}\label{thrm:stl}
The STL estimator proposed in \citep{xu2021infinitely} satisfies
\begin{align}
    \frac{\mathrm{d}}{\mathrm{d}\varepsilon}  \mathcal{F}(\rvu^* + \varepsilon \vphi) \Big\vert  _{\varepsilon = 0} =0,
\end{align}
almost surely, for all smooth and bounded perturbations $\vphi$.
\end{theorem}
The proof for the above result can be found in Appendix \ref{appdx:stl}, and combines results from \cite{nusken2021solving}.

\subsection{Structured SVI in Models with Local and Global Variables}\label{subsec:svi}

Algorithm \ref{alg:svi} produces unbiased estimates of the gradient\footnote{Gradients are computed automatically via reverse mode differentiation \citep{bartholomew2000automatic,giles2008extended} using the pytorch library \citep{paszke2019pytorch}.} as demonstrated in Appendix \ref{apdx:svi} only under the assumption that the parameters are global, that is when there is not a local parameter for each data point. In the setting where we have local and global variables we can no longer do mini-batch updates as in Algorithm \ref{alg:svi} since the energy term in the objective does not decouple as a sum over the datapoints \citep{hoffman2013stochastic, hoffman2015structured}. In this section we discuss said limitation and propose a reasonable heuristic to overcome it.

We consider the general setting where  our model has global and local variables $\bm{\Phi}, \{\vtheta_i\}$ satisfying $\vtheta_i\bigCI\vtheta_j \vert   \bm{\Phi}$ \citep{hoffman2013stochastic}. This case is particularly challenging as the local variables  scale with the size of the dataset and so will the state space. This is a fundamental setting as many hierachical latent variable models in machine learning admit such dependancy structure, such as Topic models \citep{pritchard2000stephens,blei2003latent};  Bayesian factor analysis \citep{amari1996new,bishop1999bayesian,klami2013bayesian,daxberger2019bayesian};  Variational GP Regression \citep{hensman2013gaussian}; and others. 

\begin{remark}\label{remark:fail}
    The heat semigroup does not preserve conditional independence structure in the drift, i.e. the optimal drift does not decouple and thus depends on the full state-space (Appendix \ref{apdx:decoup}).
\end{remark}

\begin{figure*}[t!]
    \centering
    \begin{minipage}{0.3\textwidth}
    \centering
    \includegraphics[width=0.9\textwidth]{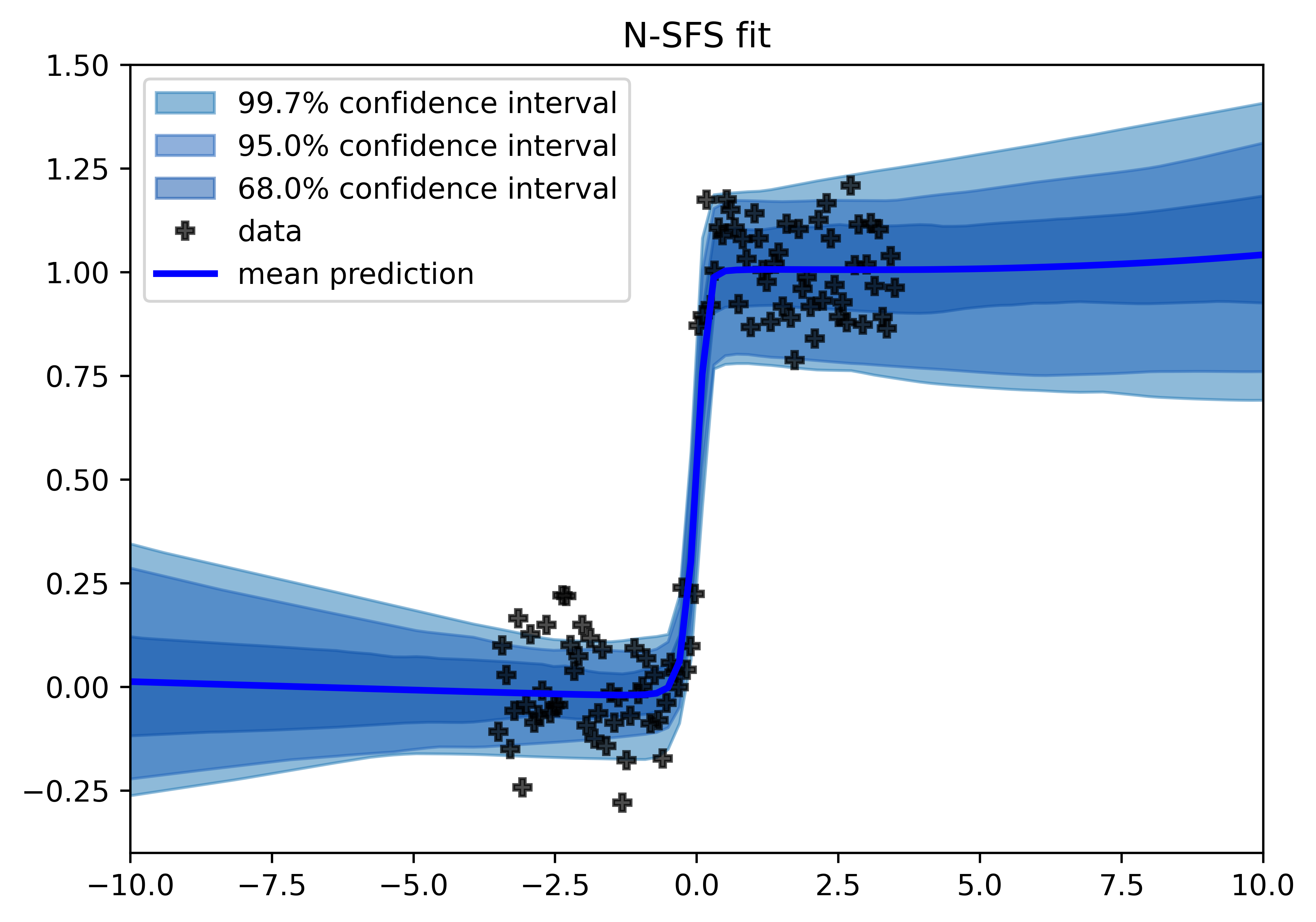}
    \end{minipage}
    \hspace*{\fill}\begin{minipage}{0.3\textwidth}
    \centering
    \includegraphics[width=0.9\textwidth]{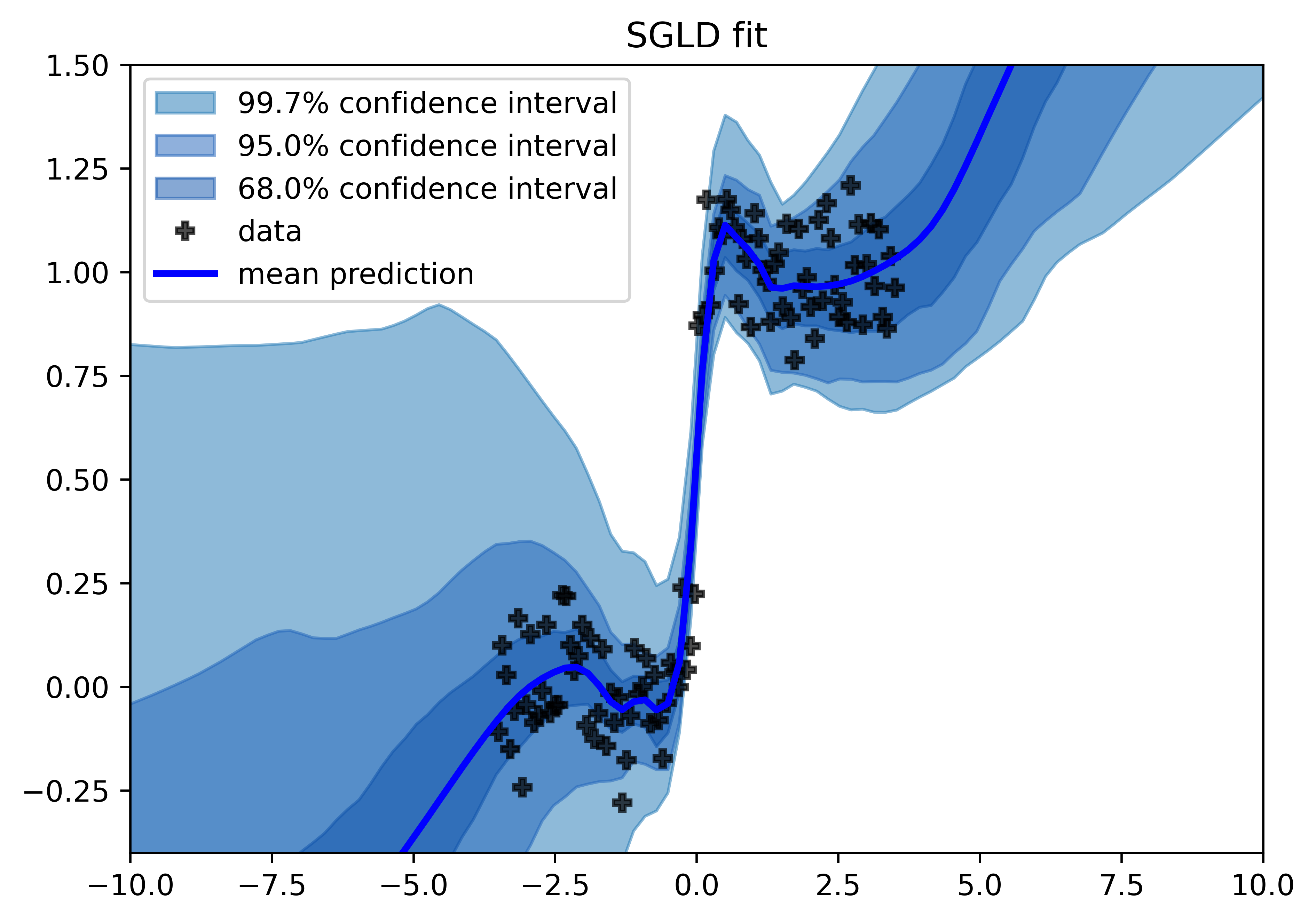}
   
    \end{minipage}
    \hspace*{\fill}\begin{minipage}{0.3\textwidth}
    \centering
    \includegraphics[width=0.9\textwidth]{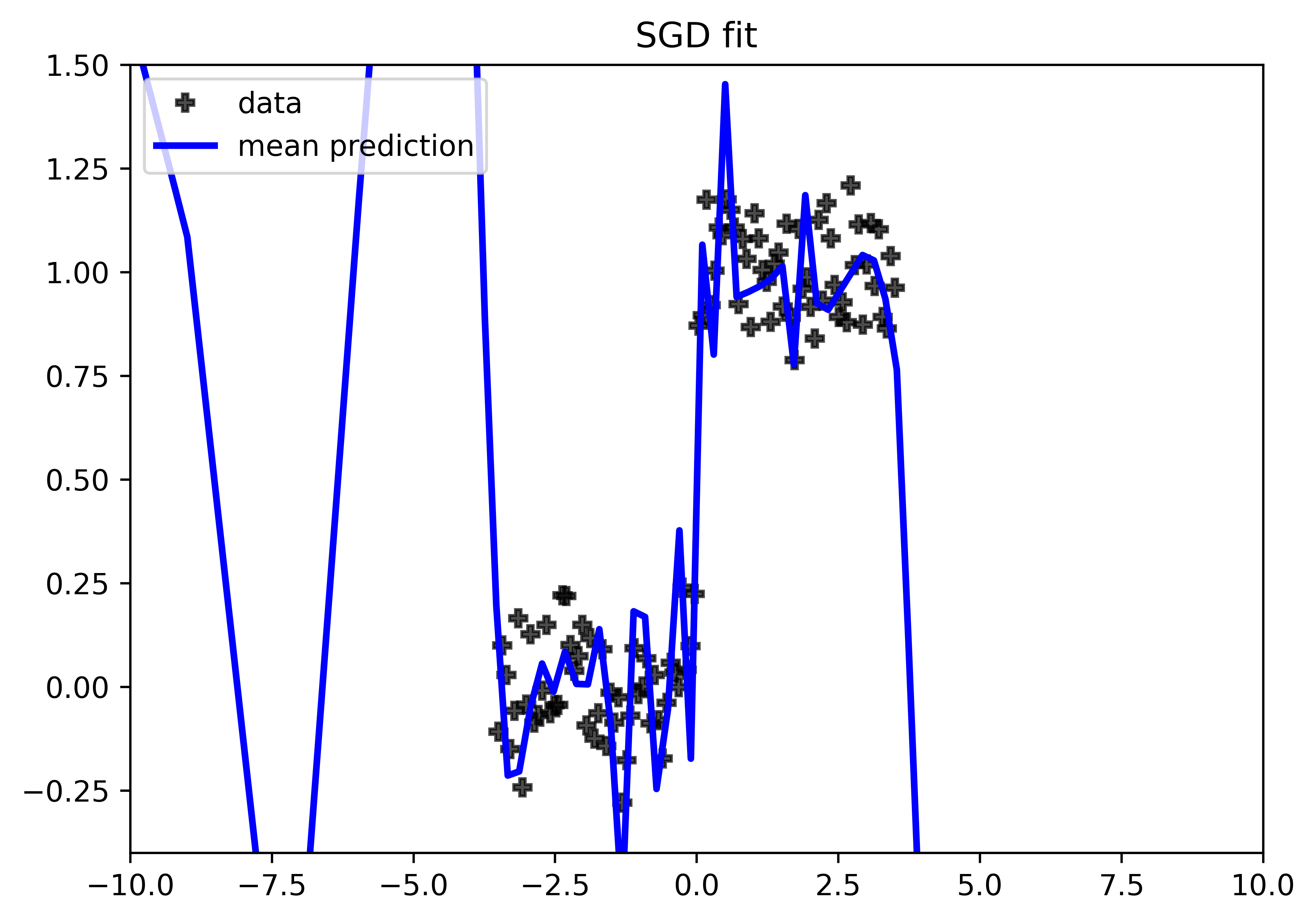}
   
    \end{minipage}
    \caption{ Visual comparison on step function data. We can see how the N-SFS based fits have the best generalisation while SGD and SGLD interpolate the noise.} \label{fig:step_functions}
\end{figure*}

Remark \ref{remark:fail} tells us that the drift is not structured in a way that admits scalable sampling approaches such as stochastic variational inference (SVI) \citep{hoffman2013stochastic}. Additionally this also highlights that the method by \cite{huang2021schrodinger} does not scale to models like this as the dimension of the state space will be linear in the size of the dataset.

In a similar fashion to \citet{hoffman2015structured}, who focussed on structured SVI, we suggest parametrising the drift via $[\rvu_{t}]_{\vtheta_i} \!\!=\!u^{\vtheta_i}(t, \vtheta_i,\Phi,\vx_i) $;
this way the dimension of the drift depends only on the respective local variables and the global variable $\Phi$. While the Föllmer drift does not admit this particular decoupling we can show that this drift is flexible enough to represent fairly general distributions, thus it is expected to have the capacity to reach the target distribution. 
Via this parametrisation we can sample in the same fashion as SVI and maintain unbiased gradient estimates.

\begin{remark}\label{rem:example}
    An SDE parametrised with a decoupled drift $[\rvu_{t}]_{\vtheta_i} =u^{\vtheta_i}(t, \vtheta_i,\Phi,\vx_i)$ can reach transition densities which do not factor (See Appendix \ref{apdx:decoup} for proof).
\end{remark}

% \begin{align}
%   \vu_\phi(\vtheta, t) = \mathrm{NN}_{1,\phi}(\vtheta,t) + \mathrm{NN}_{2,\phi}(t) \odot \nabla_{\vtheta}\ln p(\mX \vert   \vtheta)p(\vtheta)
% \end{align}

% \begin{align}
%      \mathrm{d}\rvTheta_t = \nabla_{\rvTheta_t}\ln p(\mX \vert   \rvTheta_t)p(\rvTheta_t) \mathrm{dt} + \sqrt{2}\mathrm{d}\rvw_t
% \end{align}

It is important to highlight that whilst the parametrisation in Remark  \ref{rem:example} may be flexible, it may not satisfy the previous theory developed for the Föllmer drift and SBPs, thus an interesting direction would be in recasting the SBP such that the optimal drift is decoupled. However, we found in practice that the decoupled and amortised drift worked very well, outperforming SGLD and the non-decoupled  N-SFS.

\section{Connections Between SBPs and Variational Inference in Latent Diffusion Models} \label{sec:vi}

In this section, we highlight the connection between the objective in Equation \ref{eq:CSFP} to variational inference in models with an SDE as the latent object, as studied in \cite{tzen2019neural}. 
We first start by making the connection in a simpler case -- when the prior of 
our Bayesian model is given by a Gaussian distribution with variance $\gamma$, 
that is $p(\rvtheta)=\gN(\rvtheta\vert  \bm{0}, \gamma \I)$.
\begin{observation}\label{obs1}
    When $p(\rvtheta)=\gN(\rvtheta\vert  \bm{0}, \gamma \I)$, it follows that 
    the N-SFP objective in Equation \ref{eq:CSFP} corresponds to the negative 
    ELBO of the model:

    \begin{align}
        d\rvTheta_t &= \sqrt{\gamma} d\rvw_t, \;\;\; \rvTheta_0 \sim \delta_0,  \nonumber \\
        \rvx_i &\sim p(\rvx_i \vert   \rvTheta_1) . \label{eq:model_simple_orignal}
    \end{align}
\end{observation}
% Following the recursive nature of Bayesian updates \citep{khan2021bayesian} we arrive at 
While the above observation highlights a specific connection between N-SFP and 
traditional VBI (Variational Bayesian Inference), it is limited to Bayesian models that are specified with Gaussian priors. In Lemma \ref{lem:vi} of Appendix \ref{appdx:vicon} we extend this result to more general priors and reference process via exploiting the general recursive nature of Bayesian updates \citep{khan2021bayesian}.
% To extend the result, we take inspiration  from the recursive nature of Bayesian updates  \citep{khan2021bayesian} in the following result.
% \begin{lemma}\label{lem:vi}
%     The SBP $\;\inf_{\Q \in \calD\left(\delta_0,\; p(\rvtheta\vert   \mX)\right)} \KL\left(\Q \big\vert  \big\vert   \sS\right)$
%     with reference process $\sS$ described by
%     \begin{align}
%         \rvTheta_0 &\sim \delta_{0} \\
%         d\rvTheta_t\!=\!\nabla \ln Q^\gamma_{1-t}&\left[\frac{p(\rvTheta_t)}{\gN(\rvTheta_t\vert  \bm{0}, \gamma \I)}\right] +\! \sqrt{\gamma}  d\rvw_t,  
%          \label{eq:prior_foll}
%     \end{align}
%     corresponds to maximising the ELBO of the model:
%     \begin{align}
%     \rvTheta_0 &\sim \delta_{0}\nonumber, \\
%       \;\; d\rvTheta_t\!=\!\nabla \ln Q^\gamma_{1-t}&\left[\frac{p(\rvTheta_t)}{\gN(\rvTheta_t\vert  \bm{0}, \gamma \I)}\right] +\! \sqrt{\gamma}  d\rvw_t,  ,\nonumber\\
%         \rvx_i &\sim p(\rvx_i \vert   \rvTheta_1) 
%         . \label{eq:model_simple_2}
%     \end{align}
% \end{lemma}
In short, we can view the objective in Equation \ref{eq:CSFP} as an instance of variational Bayesian inference with an SDE prior. Note that this provides a succinct connection between variational inference and maximum entropy in path space \citep{leonard2012schrodinger}. In more detail, this observation establishes an explicit connection between the ELBO of an SDE-based generative model where the SDE is latent and the SBP/stochastic-control objectives we explore in this work.

Note that Lemma \ref{lem:vi} induces a new two stage algorithm in which we first estimate a prior reference process as in Equation \ref{eq:prior_foll} and then we optimise the ELBO for the model in Equation \ref{eq:model_simple_2}. This raises the question as to what effect the dynamical prior can have within SBP-based frameworks. In practice we do not explore this formulation as the Föllmer drift of the prior may not be available in closed form and thus may require resorting to additional approximations.

\section{Experimental Results}

We ran experiments on Bayesian NN regression, classification, logistic regression and ICA \citep{amari1996new}, reporting
accuracies, log joints \citep{welling2011bayesian,izmailov2021bayesian} and expected calibration error (ECE) \citep{guo2017calibration}. For details on exact experimental setups please see Appendix \ref{apdx:exp}. Across experiments we compare to SGLD as it has been shown to be a competitive baseline in Bayesian deep learning \citep{izmailov2021bayesian}. Notice that we do not compare to more standard MCMC methodologies \citep{duane1987hybrid,neal2011mcmc,doucet2001sequential} as they do not scale well to very high dimensional tasks such as Bayesian DL \citep{izmailov2021bayesian} which are central to our experiments. However, \cite{huang2021schrodinger} contrasts the performance of the heat semigroup SFS sampler with more traditional MCMC samplers in 2D toy examples, finding SFS to be competitive \footnote{Supporting code at \url{https://github.com/franciscovargas/ControlledFollmerDrift}.}.

\begin{table}[]
        \caption{a9a dataset.}
      \centering
        \adjustbox{max width=0.9\linewidth}{\begin{tabular}{@{}llll@{}}
\toprule
Method & Accuracy          & ECE               & Log Likelihood     \\ \midrule
N-SFS  & $0.8498\pm0.0002$ & $0.0099\pm 0.0010$ & $-0.3407\pm0.0004$ \\
SGLD   & $0.8515\pm0.0010$  & $0.0010\pm0.0020$   & $-0.3247\pm0.0002$ \\ \bottomrule
\end{tabular} }\label{tab:logreg}
\end{table}

\begin{table}
    \begin{minipage}{.5\linewidth}
      \caption{Step function dataset.}
      \centering
\adjustbox{max width=0.9\linewidth}{
        \begin{tabular}{@{}lll@{}}
\toprule
Method & MSE             & Log Likelihood     \\ \midrule
N-SFS  & $0.0028\pm0.0010$ & $-63.048\pm8.2760$ \\
SGLD   & $0.1774\pm0.1280$  & $-1389.581\pm834.9680$ \\ \bottomrule
\end{tabular} \label{tab:step}}
    \end{minipage}%
    \begin{minipage}{.5\linewidth}
      \centering
        \caption{MEG dataset.}
\adjustbox{max width=0.9\linewidth}{
\begin{tabular}{@{}ll@{}}
\toprule
Method   & Log Likelihood     \\ \midrule
N-SFS   & $-5.1110\pm0.1288$ \\
SGLD    & $-4.9360\pm0.0423$ \\ \bottomrule
\end{tabular} \label{tab:ica}
}
    \end{minipage} 
\end{table}
\subsection{Bayesian Linear Regression and Comparison with MC-SFS}

In this section we explore a bayesian linear regression model with a prior on the regression weights. As this model has a Gaussian closed form for the posterior predictive distribution we report the error of the MC-SFS and N-SFS posterior predictive mean and variance with respect to the true posterior predictive moments as is seen in Figure \ref{fig:mc-sfs}. The datasets where generated by sampling the inputs randomly from a spherical Gaussian distribution and transforming them via:
\begin{align*}
    y_i = \vone ^\top \vx_i + 1
\end{align*}
we then estimated the posterior of the model:
 \begin{align*}
 \vtheta &\sim \gN(\bm{0}, \sigma^2_\theta \mathbb{I}),\\
     y_i \vert   \vx_i&,\vtheta \sim \gN(y_i\vert  \vtheta^\top (\vx_i \oplus 1), \sigma_y^2\sI), 
 \end{align*}
Where we use $\vx \oplus 1$ to denote adding an extra dimension with a $1$ to the vector $\vx$. We carried out this experiment increasing the dimension of $\vx$ from $2^5$ to $2^{11}$. We observe that the N-SFS based approaches have overall a notably smaller posterior predictive error to the MC-SFS approach. Finally we note the STL method is more concentrated in its predictions than the naive N-SFS approach, whilst having similar errors. 

\subsection{Bayesian Logistic Regression / Independent Component Analysis - a9a / MEG Datasets}

Following \cite{welling2011bayesian} we explore a logistic regression model on the a9a dataset. Results can be found in Table \ref{tab:logreg} which show that N-SFS achieves a test accuracy, ECE and log likelihood comparable to SGLD. We then explore the performance of our approach on the Bayesian variant of ICA studied in \cite{welling2011bayesian} on the MEG-Dataset \citep{vigario1997meg}.  We can observe (Table \ref{tab:ica}) that here N-SFS also achieves results comparable to SGLD.

\subsection{Bayesian Deep Learning}
 In these tasks we use models of the form
 \begin{align*}
 \vtheta &\sim \gN(\bm{0}, \sigma^2_\theta \mathbb{I}),\\
     \vy_i \vert   \vx_i&,\vtheta \sim p(\vy_i\vert  f_\vtheta(\vx_i)), 
 \end{align*}
 where $f_\vtheta$ is a neural network. In these settings we are interested in using the posterior predictive distribution $p(\vy^* \vert   \vx^*, \mX) \!=\!\int p(\vy^*\vert  f_\vtheta(\vx^*)) dP(\vtheta\vert   \mX) $ to make robust predictions. Across the image experiments we use the LeNet5 \citep{lecun1998gradient} architecture. Future works should explore recent architectures for images such as VGG-16 \citep{simonyan2014very} and ResNet32 \citep{he2016deep}.

\noindent  \textbf{Non-linear Regression - Step Function:} We fit a 2-hidden-layer neural network with a total of 14876 parameters on a toy step function dataset. We can see in Figure \ref{fig:step_functions} how both the SGD and SGLD fits interpolate the noise, whilst N-SFS has straight lines, thus both achieving a better test error and having well-calibrated error bars. We believe it is a great milestone to see how an overparameterised neural network is able to achieve such well calibrated predictions. 

\begin{table*}[]
\centering \caption{Test set results on MNIST, Rotated MNIST and CIFAR10. The Log-likelihood column is the mean posterior predictive and is thus not estimated for SGD.}
\begin{tabular}{@{}ccccc@{}}
\toprule
Dataset                        & Method                 & Accuracy          & ECE               & Log Likelihood                 \\ \midrule
\multirow{4}{*}{MNIST}         & N-SFS                  & $0.9889\pm0.0013$ & $0.0080\pm0.0013$ & $-0.0883\pm0.0076$ \\
                               & N-SFS$_{\mathrm{stl}}$ & $0.9885\pm0.0014$ & $0.0092 \pm0.0017$ & $-0.0629\pm0.0057$ \\
                               & SGLD                   & $0.9837\pm0.0007$ & $0.0061\pm0.0012$ & $-0.0516\pm0.0026$ \\
                               & SGD                    & $0.9884\pm0.0007$ & $0.0034\pm0.0009$ & - \\
                               &                        &                   &                   &                    \\
\multirow{4}{*}{Rotated-MNIST} & N-SFS                  & $0.9479\pm0.0043$ & $0.0077\pm0.0012$ & $-0.3890\pm0.0374$ \\
                               & N-SFS$_{\mathrm{stl}}$ & $0.9461\pm0.0039$ & $0.0057\pm0.0012$ & $-0.2960\pm0.0336$ \\
                               & SGLD                   & $0.9247\pm0.0035$ & $0.0141\pm0.0018$ & $-0.2439\pm0.0118$ \\
                               & SGD                    & $0.9404\pm0.0031$ & $0.0284\pm0.0021$ & - \\
                               &                        &                   &                   &                    \\
\multirow{4}{*}{CIFAR10}       & N-SFS                  & $0.6156\pm0.0021$ & $0.0520\pm0.0110$ & $-1.3628\pm0.0262$ \\
                               & N-SFS$_{\mathrm{stl}}$ & $0.6264\pm0.0286$ & $0.0568\pm0.0069$ & $-1.2305\pm0.0710$ \\
                               & SGLD                   & $0.6232\pm0.0186$ & $0.1493\pm0.0170$ & $-1.2740\pm0.0854$ \\
                               & SGD                    & $0.6229\pm0.0124$ & $0.0626\pm0.0163$ & - \\ \bottomrule
\end{tabular} \label{tab:dl_class}
\end{table*}

\noindent \textbf{Digits Classification - LeNet5:} We train the standard LeNet5 \citep{lecun1998gradient} architecture (with 44426 parameters) on the MNIST dataset \citep{mnisthandwrittendigit}. At test time we evaluate the methods on the MNIST test set augmented by random rotations of up to 30\degree \citep{ferianc2021effects}.  Table \ref{tab:dl_class} shows how N-SFS has the highest accuracy whilst obtaining the lowest calibration error among the considered methods, highlighting that our approach has the most well-calibrated and accurate predictions when considering a slightly perturbed test set. We highlight that LeNet5 falls into an interesting regime as the number of parameters is considerably less than the size of the training set, and thus we can argue it is not in the overparameterised regime. This regime \citep{belkin2019reconciling} has been shown to be  challenging in achieving good generalisation errors, thus we believe the predictive and calibrated accuracy achieved by N-SFS is a strong milestone.

Additionally we provide results on the regular MNIST test set. We can observe that N-SFS maintains a high test accuracy and at the same time preserves a low ECE score. We believe the reason SGD and SGLD obtain slightly better ECE performances is that the MNIST test set has very little variation to the MNIST training set, and thus all results seem well calibrated. We can see this observation confirmed by how the distribution of ECE scores changes dramatically on the Rotated MNIST set, a similar argument to that developed in \cite{ferianc2021effects}. We note that across both experiments SGLD achieves a slightly better log likelihood which comes at the cost of lower predictive performance and less calibrated predictions.

\noindent \textbf{Image Classification - CIFAR10:} We fit a variation of the LeNet5 (Appendix \ref{apdx:bnn}) architecture with 62006 parameters on the CIFAR10 dataset \citep{krizhevsky2009learning}. We note that the predictive test accuracies and log-likelihoods of N-SFS$_{\mathrm{stl}}$, SGLD and SGD are comparable. However, we can see that N-SFS$_{\mathrm{stl}}$ has an ECE an order of magnitude smaller. We notice that the STL estimator made a significant difference on CIFAR10, making the training faster and more stable.

\subsection{Hyperspectral Image Unmixing}

To assess our method's performance visually, we use it to sample from Hyperspectral
Unmixing Models~\citep{Bioucas-Dias2012}. Hyperspectral images are high spectral resolution but low spatial
resolution images typically taken of vast areas via satellites. High spectral resolution
provides much more information about the materials present in each pixel. However, 
due to the low spatial resolution, each pixel of an image can correspond to a 
$50m^2$ area, containing several materials. Such pixels will therefore have mixed
and uninformative spectra. The task of Hyperspectral Unmixing is to determine
the presence of given materials in each pixel.

\begin{figure}[t!]
    \centering
    \hspace*{\fill}\begin{minipage}{0.2\textwidth}
        \centering
        \includegraphics[width=\textwidth]{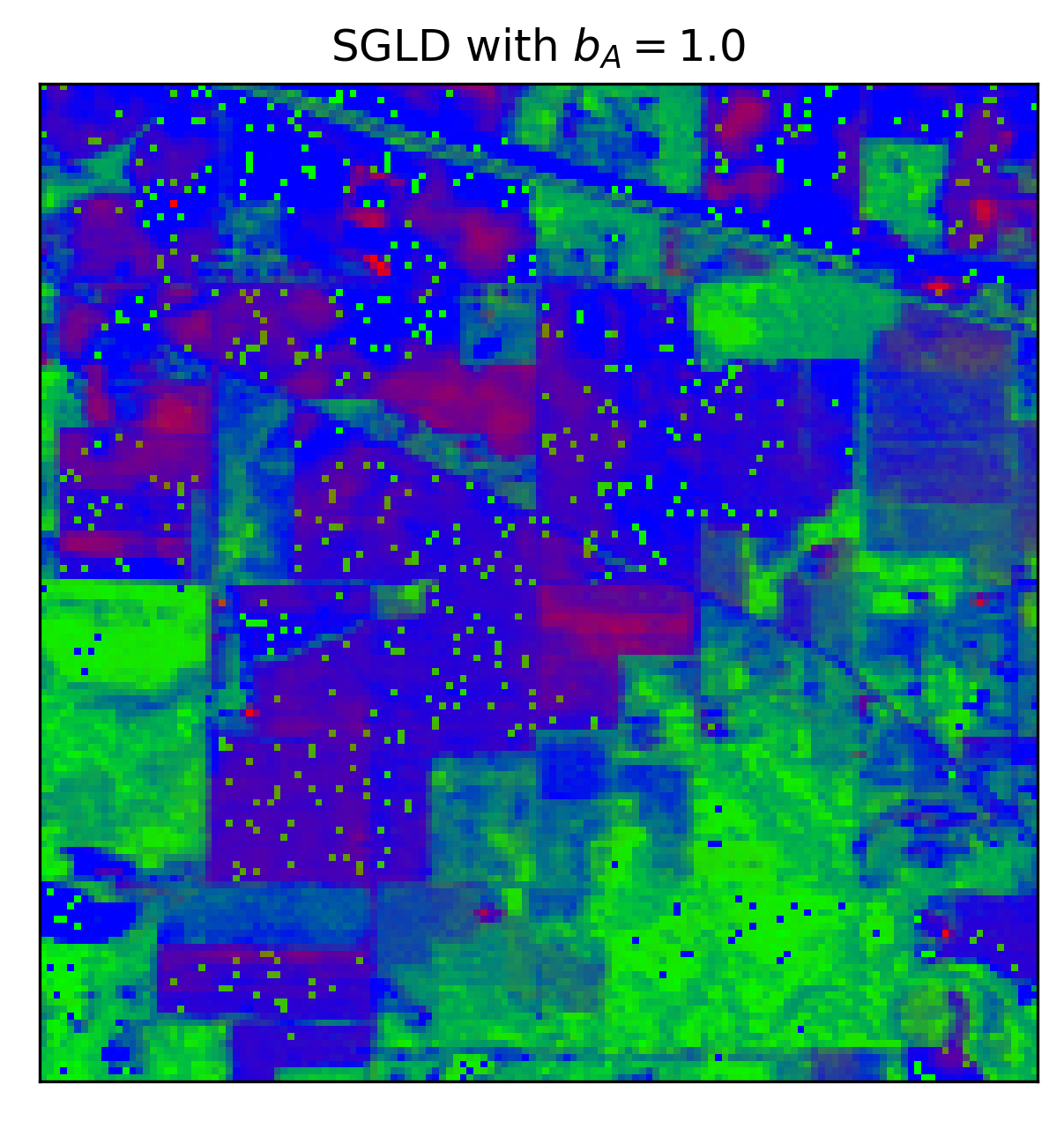}
    \end{minipage}
    \hspace*{\fill}\begin{minipage}{0.2\textwidth}
        \centering
        \includegraphics[width=\textwidth]{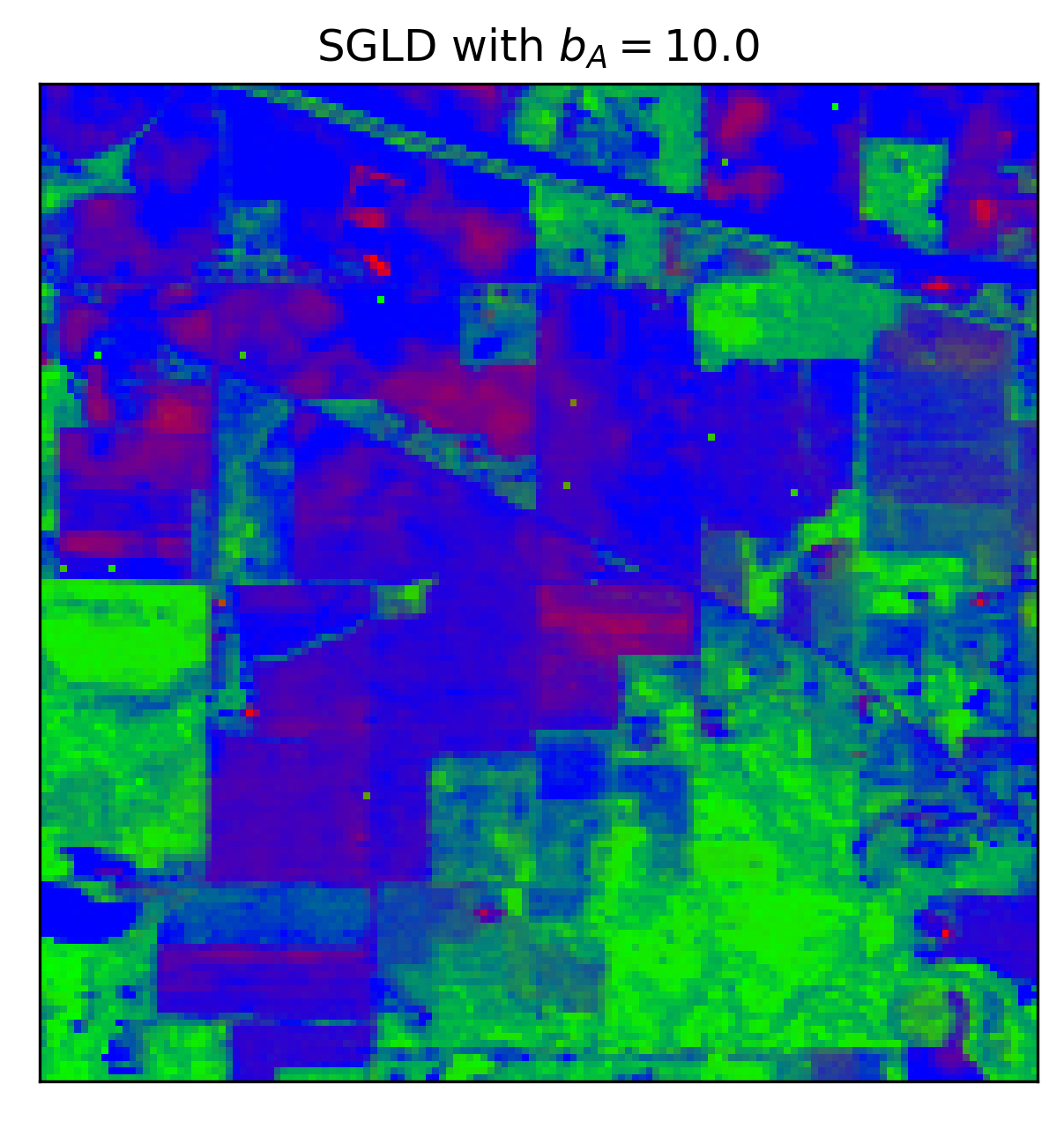}
    \end{minipage}
    \hspace*{\fill}\begin{minipage}{0.2\textwidth}
        \centering
        \includegraphics[width=\textwidth]{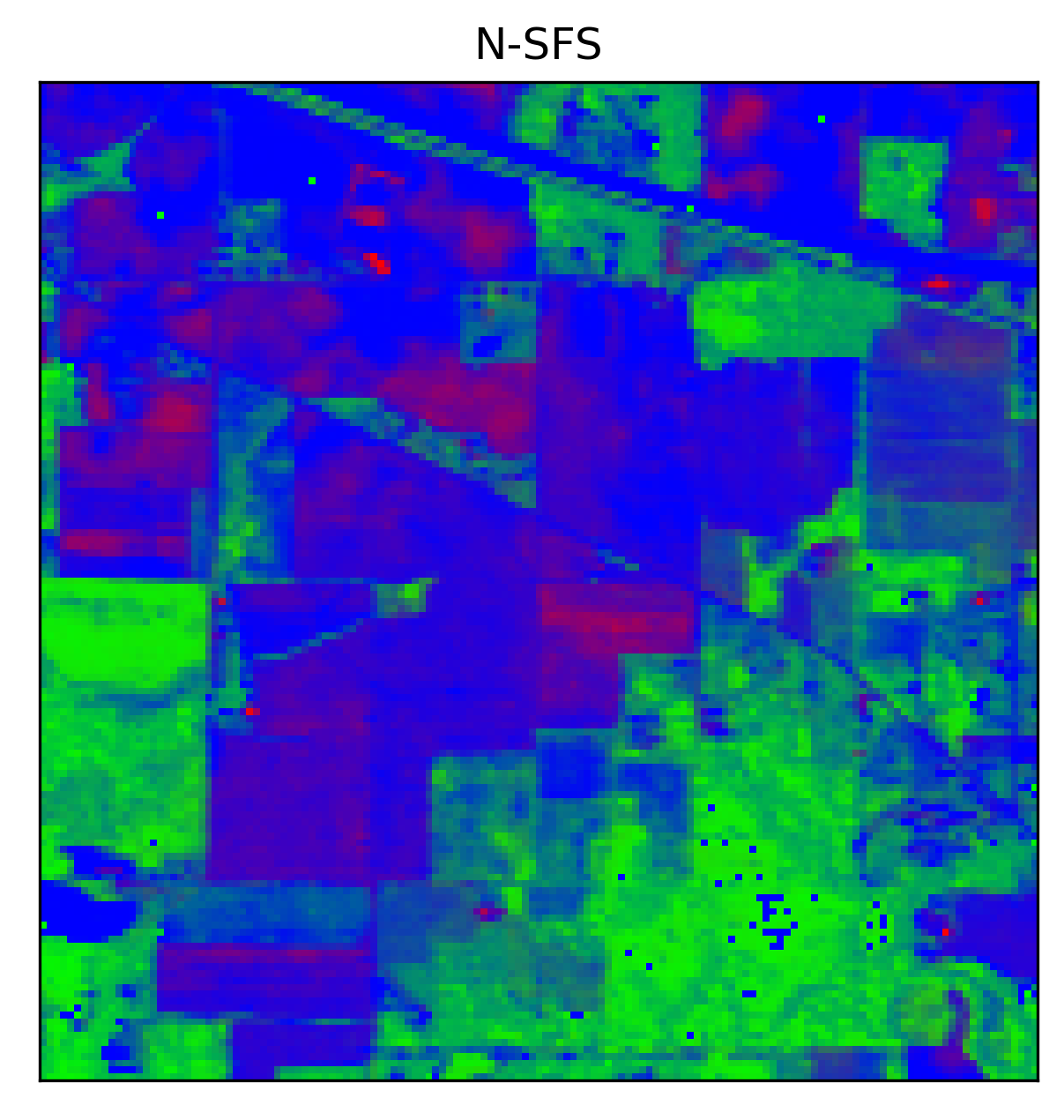}
    \end{minipage}
    \hspace*{\fill}\begin{minipage}{0.2\textwidth}
        \centering
        \includegraphics[width=\textwidth]{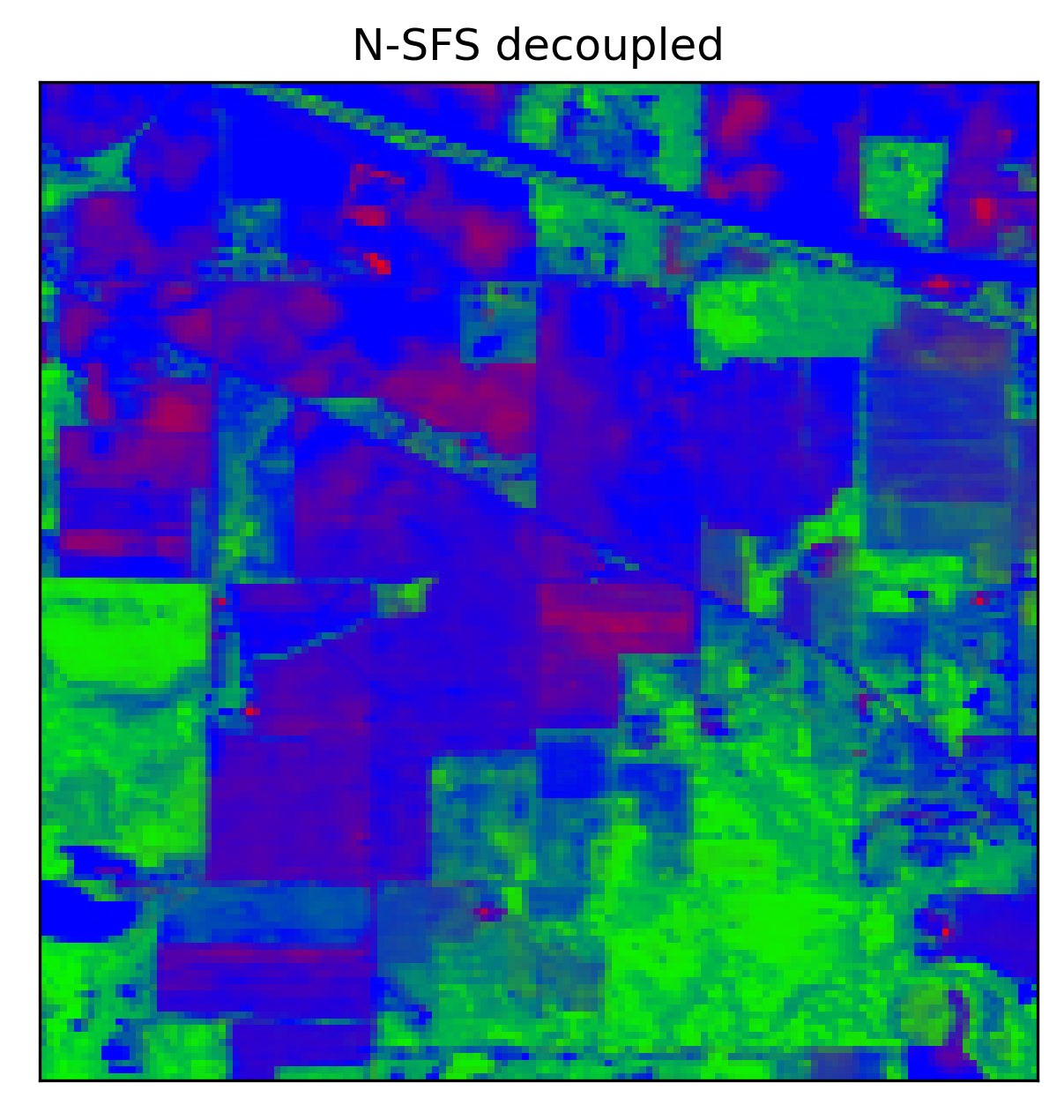}
    \end{minipage}

    \caption{False-color composites with channels given by the unmixed matrices $\mA$
        obtained via SGLD, N-SFS and N-SFS with a decoupled drift. Speckles illustrate mode collapse.} \label{fig:unmixing_results}
\end{figure}

We use the Indian Pines image\footnote{taken from \url{http://www.ehu.eus/ccwintco/index.php/Hyperspectral_Remote_Sensing_Scenes}}, 
denoted as $\bY$, which has a spatial resolution of
$P = 145 \times 145 = 21025$ pixels and a spectral resolution of $B = 200$ bands, i.e.\ $\bY = [\by_1, \dots, \by_P] \in [0, 1]^{B \times P}$.
$R = 3$ materials have been chosen automatically using the Pixel Purity Index and 
the collection of their spectra will be denoted as $\bM = [\bmm_1, \bmm_2, \bmm_3] \in [0, 1]^{B \times 3}$.
The task of Hyperspectral Unmixing is to determine for each pixel $p$ a vector 
$\ba_p \in \psimplex{R}$ in the probability simplex, where
$[\mA ]_{p,i}=a_{p, i}$ represents the fraction of the $i$-th material in pixel $p$.
To determine the presence of each material, we use the Normal Compositional Model
~\citep{Eches2010}
as it is a challenging model to sample from. Specifically, it has parameters $(\mPhi, \rvTheta) = (\sigma^2, \bA)$
and is defined by:
\begin{align*}
    p\left(\sigma^2\right) = \ind_{[0, 1]}\left(\sigma^2\right), \quad p\left(\bA\right) = \prod_{p=1}^P \ind_{\psimplex{R}} \left(\ba_p\right),\\
     p\left(\bY \vert   \bA, \sigma^2\right) = \prod_{p=1}^P \gN\left(\by_p; \bM\ba_p; \vert  \vert  \ba_p\vert  \vert  ^2 \sigma^2 I\right),
\end{align*}
First note that this model follows the structured model setting discussed in
Section 2.2 --- it has one global parameter $\sigma^2$ and a local parameter $\ba_p$
for each pixel.  
Finally, while all the parameters are constrained to lie
on the probability simplices, this sampling problem can be cast into an unconstrained
sampling problem via Lagrange transformations as in~\cite{Hsieh2018}.
The Normal Compositional Model~\cite{Eches2010} is primarily of interest to us
because the unusual noise scaling in the likelihood can produce several modes in each
pixel, making it especially easy for sampling algorithms to get stuck in modes.

We compared three approaches for this problem: 1) SGLD
2) N-SFS 3) N-SFS with decoupled drift, where the decoupled drift is defined as:
\begin{equation*}
\resizebox{0.902\columnwidth}{!}{$\displaystyle
    \vu_t(\sigma^2, \bA\!)\! =\! [u_0(t, \sigma^2), u_1(t, \sigma^2, \ba_1\!), \dots, u_P(t, \sigma^2, \ba_P\!)].
$}
\end{equation*}
Unmixing results are shown in Figure~\ref{fig:unmixing_results}. 
We stress that to run SGLD successfully we had to tune the approach heavily --- we 
used separate step sizes (which acts as a preconditioning) and step size schedules 
for parameters $\sigma^2$ and $\bA$, only with one combination of which we managed 
to get decent unmixing results.
Without the amortised drift, N-SFS struggled with multiple modes
in certain patches of the image, however, decoupling the drift resulted in almost
perfect unmixing.
With a slight deviation from the optimal step size schedule, SGLD fails to explore modes and produces speckly images. In contrast, the only tunable parameter for N-SFS 
was $\gamma$, which was giving similar results for all tried values. Further
sensitivity results for SGLD/N-SFS are provided in Appendix \ref{apdx:sens}.

\subsection{Analysis of N-SFS training dynamics}

\begin{figure*}[ht!]
    \centering
    \includegraphics[width=0.7\textwidth]{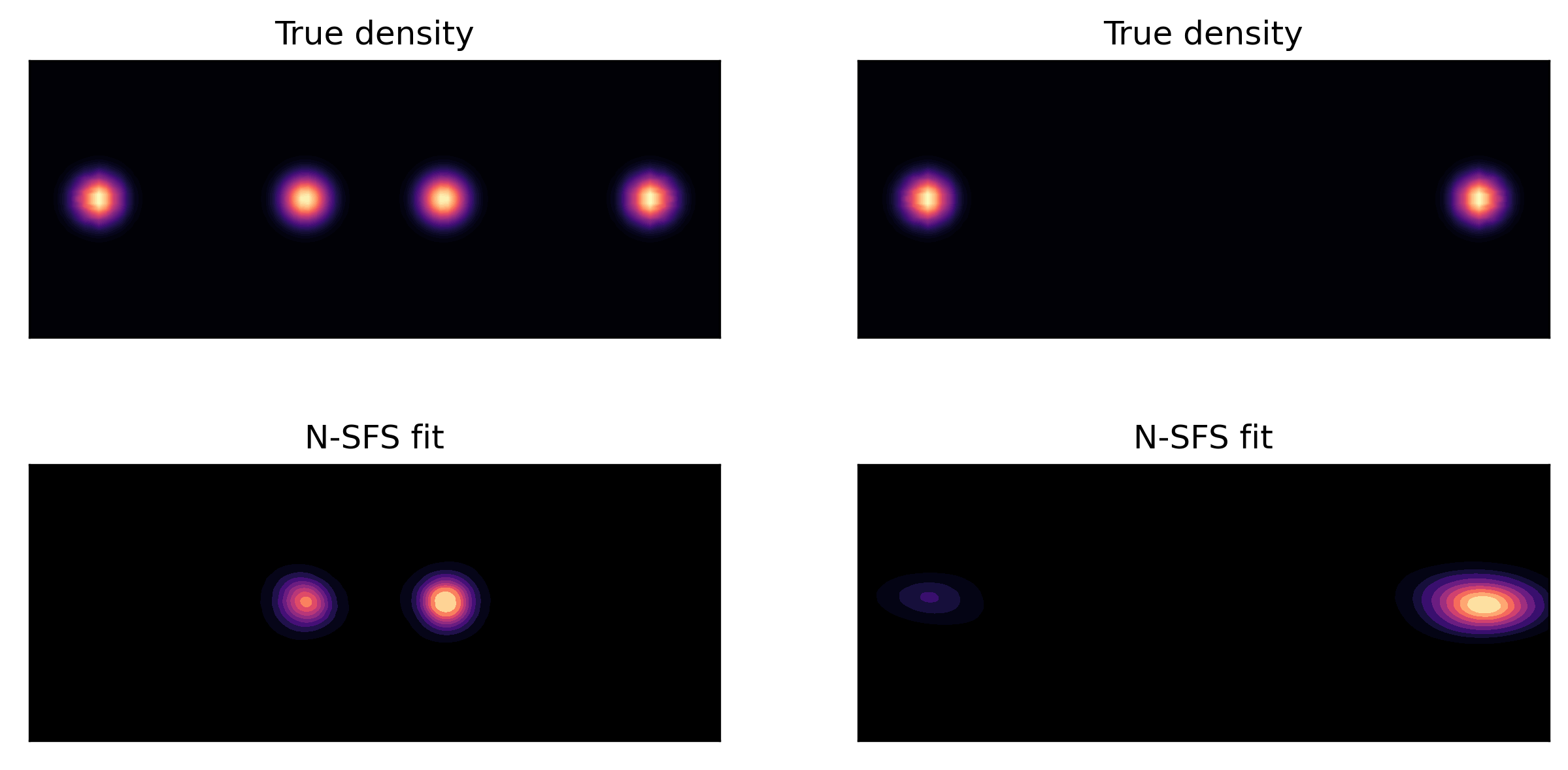}
    \caption{N-SFS performance on a gaussian mixture posterior distribution with several modes. Outer modes are only detected when the posterior does not contain the interior modes indicating exploration failure of N-SFS.}
    \label{fig:nsfs-modes}
\end{figure*}

In addition to the experiments above, we investigate our method's performance in a synthetic multi-modal scenario. Here, N-SFS is used to fit a Gaussian Mixture posterior distribution that has modes aligned on the $x$-axis, as shown in figure~\ref{fig:nsfs-modes}. In one case, there are 4 modes -- 2 inner modes (those closer to $0$) and 2 outer modes (those further away from $0$). We notice that in the presence of the 2 inner modes N-SFS is unable to discover the outer modes. In contrast, when considering a posterior with only the 2 outer modes, the distribution is fit correctly. This phenomenon could be explained by previously indicated connections between stochastic control and agent-based learning via the Hamilton-Jacobi-Bellman equation~\cite{powell2019} and the exploration-exploitation tradeoff. More concretely, the optimisation objective~\eqref{eq:CSFP} implies the following training dynamics -- random samples are generated from a diffusion (a Brownian motion to begin with) which is then refined to produce more samples in areas where previous samples had high posterior density. This implies that after some modes are discovered, the diffusion will be adjusted to fit them, i.e. the algorithm immediately starts exploiting the detected modes. Other modes will only be discovered if some random sample accidentally hits them, which is very unlikely if the modes are far away. This indicates that the algorithm could be improved by incorporating exploration techniques found in agent-based learning literature.

\begin{figure*}[ht!]
    \centering
    
    \begin{minipage}{0.35\linewidth}
      \includegraphics[width=1.0\textwidth]{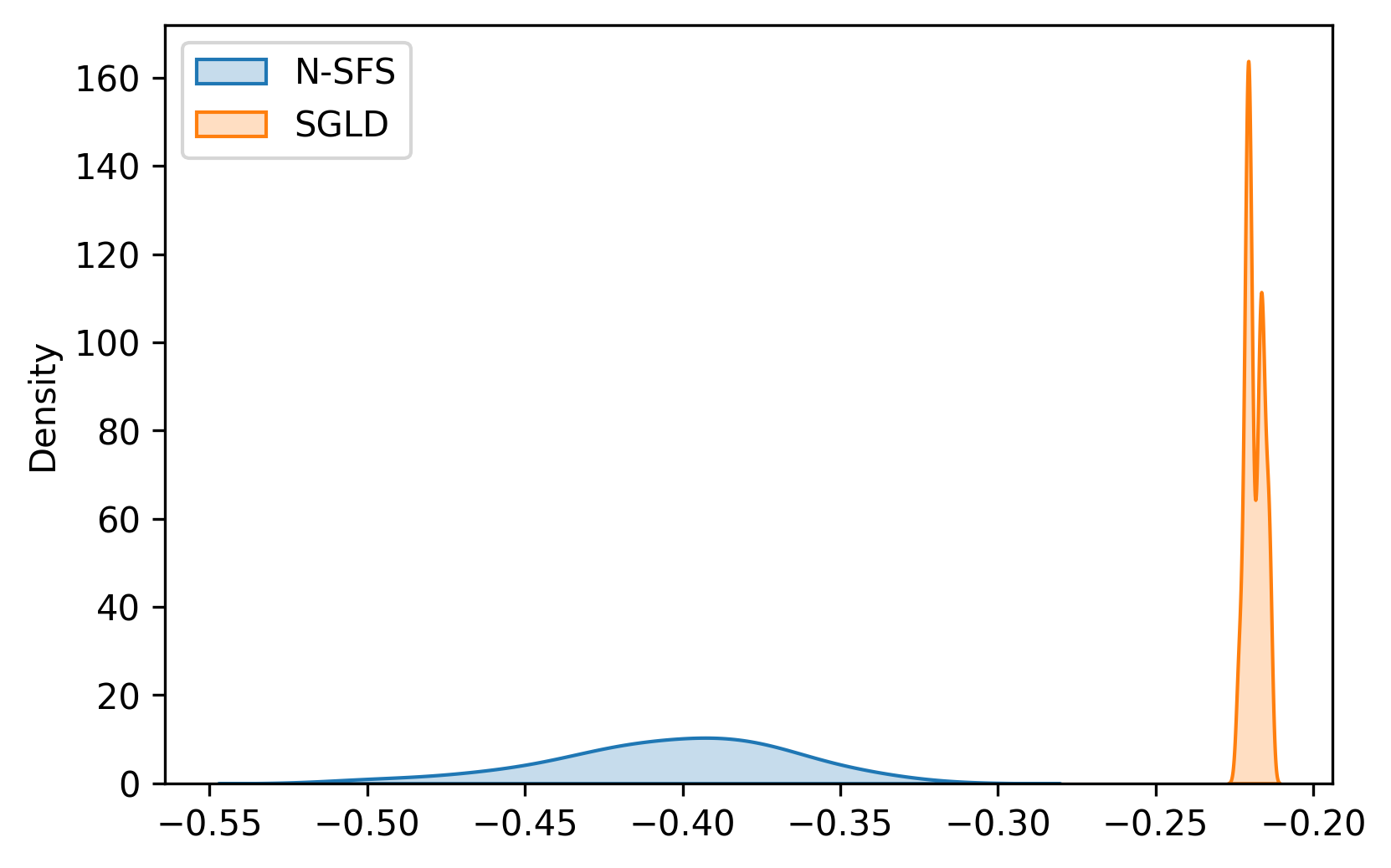}
    \end{minipage}
    \hspace*{0.1\linewidth}
    \begin{minipage}{0.35\linewidth}
      \includegraphics[width=1.0\textwidth]{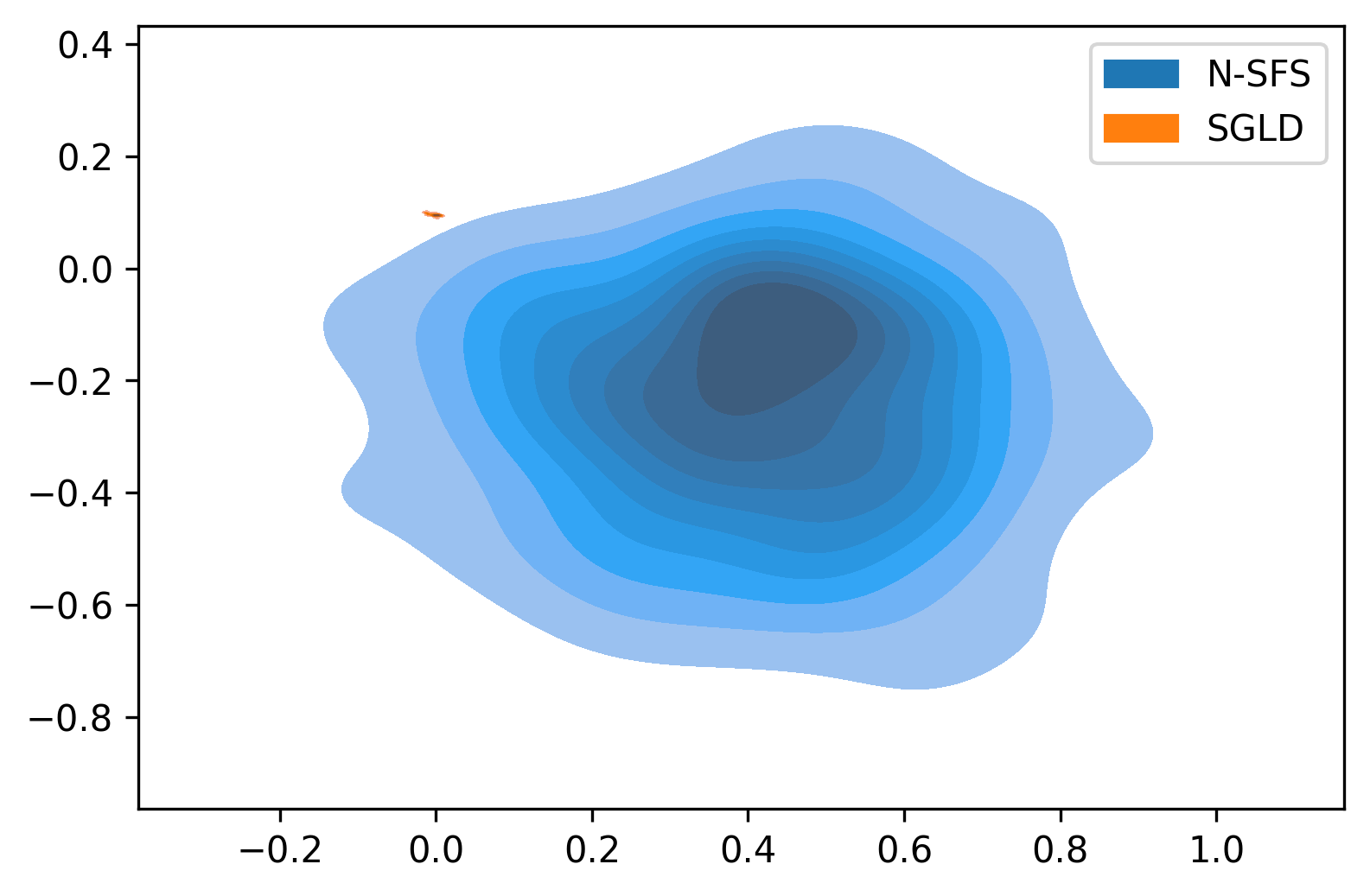}
    \end{minipage}
    
    \caption{Distribution of log posterior values of samples from N-SFS and SGLD \textbf{(left)} and marginal distribution of a pair of weights in a neural network obtained from samples of N-SFS and SGLD \textbf{(right)}}
    \label{fig:mnist-weight-dist}
\end{figure*}

Given the behaviour of N-SFS on this multi-modal example, it is then natural to ask if it happens in Bayesian Deep Learning applications. To examine this, we look at the marginal distributions of a pair of weights of a Bayesian Neural Network for MNIST classification given by the samples of N-SFS and SGLD given in Figure~\ref{fig:mnist-weight-dist}. Note that compared to SGLD, N-SFS samples from a dramatically wider distribution, while maintaining a comparable predictive log likelihood score, and therefore does not suffer from the lack of exploration.

\section{Discussion and Future Directions}

Overall we achieve predictive performance competitive to SGLD across a variety of tasks whilst obtaining better calibrated predictions as measured by the ECE metric. We hypothesise that the gain in performance is due to the flexible and low variance VI parametrisation of the proposed approach. We would like to highlight that these results were achieved with minimal tuning and simple NN architectures. We find that the decoupled and amortised drift we propose achieves very strong results making our approach tractable to Bayesian models with local and global structure. Additionally we notice that the architecture used in the drift network can influence results, thus future work in this area should develop the drift architectures further. 

A key advantage of our approach is that at training time the objective  effectively minimises an ELBO styled objective parameterised via a ResNet. This allows us to monitor training using the traditional techniques from deep learning, without the challenges arising from mixing times and correlation of samples found in traditional MCMC methods; once N-SFS is trained, generating samples at test time is a fast forward pass through a ResNet that does not require re-training. Finally, as we demonstrated, our approach allows the learned sampler to be amortised \citep{zhang2018advances} which not only allows the drift to be more tractably parameterised but also creates the prospects of meta learning the posterior \citep{edwards2016towards,yoon2018bayesian,gordon2018meta,gordon2021advances}. We believe that this work motivates how stochastic control paves a new exciting and promising direction in Bayesian ML/DL. 

% \ifnum\arxiv=1 \\\\

% \\\\

\noindent \textbf{Acknowledgements.}  Francisco Vargas is Funded by Huawei Technologies Co. This research has been partially funded by Deutsche Forschungsgemeinschaft (DFG)
through the grant CRC 1114 ‘Scaling Cascades in Complex Systems’ (project A02, project number 235221301).
Andrius Ovsianas is funded by EPSRC iCASE Award EP/T517677/1. Mark Girolami is 
supported by a Royal Academy of Engineering Research Chair, and EPSRC grants 
EP/T000414/1, EP/R018413/2, EP/P020720/2, EP/R034710/1, EP/R004889/1.
% \fi

\begin{appendices}
\onecolumn

\section{Main Results} 

\subsection{Posterior Drift}

\begin{repcorollary}{col:main}
The minimiser
\begin{align}
  \rvu^{*}=\argmin_{\rvu \in \gU}\E_{\rvTheta \sim \Q^{\rvu, \delta_0}}\left[\frac{1}{2\gamma}\int_0^1\|\rvu(t, \rvTheta_t)\|^2dt-\ln\left(\frac{ p(\mX \vert   \rvTheta_1)p(\rvTheta_1)}{\gN(\rvTheta_1\vert  \bm{0}, \gamma \I)}\right)\right]
\end{align}
satisfies $\law \rvTheta_1^{\rvu^{*}} = \frac{p(\mX\vert   \rvtheta)p(\rvtheta)}{\gZ}$.

\end{repcorollary}
\begin{proof}
 This follows directly after substituting the Radon-Nikodym derivative between the Gaussian distribution and the posterior into Theorem 1 in \cite{tzen2019theoretical} or Theorem 3.1 in \cite{dai1991stochastic}.
\end{proof}

\subsection{EM-Discretisation Result} \label{app:em}

First we would like to introduce the following auxiliary theorem from \cite{tzen2019theoretical}:

\begin{theorem}\label{thrm:tzen} \citep{tzen2019theoretical} Given the standard regularity assumptions presented for $f=\frac{d\pi_1}{d\gN(\bm{0}, \gamma\sI)}$ in \cite{tzen2019theoretical}, let $L= \max \{\mathrm{Lip}(f),  \mathrm{Lip}(\nabla f)\}$ and assume that there exists a constant $c \in (0,1]$ such that $f \ge c$. Then for any $\epsilon \in \left(0, 16 \frac{L^2}{c^2}\right)$ there exists a neural net $\vv : \sR
^d \times [0, 1] \rightarrow \sR^d$ with size polynomial in $1/\epsilon, d, L, c, 1/c, \gamma$, such that the activation function of each neuron follows the regularity assumptions in \cite{tzen2019theoretical} (e.g. $\mathrm{ReLU}, \mathrm{Sigmoid}, \mathrm{Softplus}$)
and 
\begin{align}
    \KL(\pi_1\vert  \vert   \pi^{\vv}_1 ) \leq \epsilon,
\end{align}
where $\pi^{\vv}_1=\law(\rvTheta_1^\vv)$ is the terminal distribution of
the diffusion process
% governed by the Ito SDE
\begin{align}
    d\rvTheta_t^{\vv} = \vv(\rvTheta_t^\vv, \sqrt{1-t})dt + \sqrt{\gamma} d\rvw_t, \qquad t \in [0,1].
\end{align}
%then it follows that:
%\begin{align}
%    \KL(\pi_1\vert  \vert   \pi^{\vv}_1 ) \leq \epsilon
%\end{align}
%where $\pi^{\vv}_1=\law(\rvTheta_1^\vv)$ and $c \in (0,1]$ is a constant that lower bounds $f\geq c$ everywhere.
\end{theorem}
We can now proceed to prove the direct corollary of the above theorem when using the EM scheme for simulation.

\begin{repcorollary}{col:euler}
Given the network $\vv$ from Theorem \ref{thrm:tzen} it follows that the Euler-Mayurama discretisation $\hat{X}_t^\vv$ of $X_t^\vv$ has a KL-divergence to the target distribution $\pi_1$ of:
\begin{align}
     \KL(\pi_1\vert  \vert   \hat{\pi}^{\vv}_1 ) \leq\left (\epsilon^{1/2} + \mathcal{O}(\sqrt{\Delta t}) \right)^2
\end{align}
\end{repcorollary}

\begin{proof}
Consider the path-wise KL-divergence between the exact Schrödinger-Föllmer process and its EM-discretised neural approximation:
\begin{align}
    \KL(\sP^{\rvu^*}\vert  \vert  \sP^{\hat{\vv}}) = \frac{1}{2 \gamma}\int_0^1\E_{\rvTheta \sim \Q^{\rvu*, \delta_0}}\left\|\rvu^*(\rvTheta_t,t)-\hat{\vv}(\rvTheta_t, \sqrt{1-t})\right\|^2dt.
\end{align}
Defining $d(\vx, \vy) := \sqrt{\frac{1}{2 \gamma}\int_0^1\E_{\rvTheta \sim \Q^{\rvu*, \delta_0}}\left\|\vx(\rvTheta_t, t)-\hat{\vy}(\rvTheta_t, t)\right\|^2dt}$, 
it is clear that $d(\vx, \vy)$  satisfies the triangle inequality as it is the $\gL^2({\sQ^{\gamma, \rvu^*, \delta_0}})$ metric between drifts, thus applying the triangle inequality at the drift level we have that (for simplicitly letting $\gamma=1$):
\begin{align*}
    d(\rvu^*, \hat{\vv}) &\leq \left(\int_0^1\E \left[\vert  \vert  \vu_t^* - \vv_{\sqrt{1-t}}\vert  \vert  ^2\right]dt\right)^{\frac{1}{2}} +  \left(\int_0^1\E\left\vert  \vert  \vv_{\sqrt{1-t}} - \hat{\vv}_{\sqrt{1-t}}\vert  \vert  ^2\right] dt \right)^{\frac{1}{2}}.
\end{align*}
From \cite{tzen2019theoretical} we can bound the first term resulting in:
\begin{align*}
    d(\rvu^*, \hat{\vv}) \leq\epsilon^{1/2} +  \left(\int_0^1 \E\left[\vert  \vert  \vv_{\sqrt{1-t}} - \hat{\vv}_{\sqrt{1-t}}\vert  \vert  ^2\right] dt \right)^{\frac{1}{2}}
\end{align*}
Now remembering that the EM drift is given by $\hat{\vv}_{\sqrt{1-t}}(\rvTheta_t) = \vv(\hat{\rvTheta}_t, \sqrt{1-\Delta t \lceil t/\Delta t\rceil})$, we can use that $\vv$ is L'-Lipschitz in both arguments, thus:
\begin{align*}
    d(\rvu^*, \hat{\vv}) &\leq\epsilon^{1/2} + \left(L'^2\int_0^1\E\left[\left(\left\|\rvTheta_t - \hat{\rvTheta}_t\right\| + \Delta t\right)^2\right]dt \right)^{\frac{1}{2}} \\
    &\leq \epsilon^{1/2} +  \left(2L'^2\left(\E\left[\int_0^1\left\|\rvTheta_t - \hat{\rvTheta}_t\right\|^2dt\right]+\Delta t^2\right) \right)^{\frac{1}{2}} \\
    &\leq \epsilon^{1/2} +  \left(2L'^2\left(\E\left[ \max_{0\leq t\leq 1}\left\|\rvTheta_t - \hat{\rvTheta}_t\right\|^2\right] + \Delta t^2\right) \right)^{\frac{1}{2}},
\end{align*}
which, using the strong convergence of the EM approximation \citep{gyongy1996existence}, implies:
\begin{align}
    \E\left[ \max_{0\leq t\leq 1} \left\|\rvTheta_t - \hat{\rvTheta}_t\right\|^2\right] \leq C_{L'} \Delta  t,
\end{align}
thus:
\begin{align*}
    d(\rvu^*, \hat{\vv}) \leq\epsilon^{1/2} +  L'\sqrt{2}\left(\sqrt{C_{L'}\Delta t} + \Delta t \right).
\end{align*}
Squaring both sides and applying the data processing inequality completes the proof.
\end{proof}

\section{Connections to VI} \label{appdx:vicon}

We first start by making the connection in a simpler case -- when the prior of 
our Bayesian model is given by a Gaussian distribution with variance $\gamma$, 
that is $p(\rvtheta)=\gN(\rvtheta\vert  \bm{0}, \gamma \I)$.

\begin{repobservation}{obs1}
    When $p(\rvtheta)=\gN(\rvtheta\vert  \bm{0}, \gamma \I)$, it follows that 
    the N-SFP objective in Equation \ref{eq:CSFP} corresponds to the negative 
    ELBO of the model:

    \begin{align}
        d\rvTheta_t &= \sqrt{\gamma} d\rvw_t, \;\;\; \rvTheta_0 \sim \delta_0,  \nonumber \\
        \rvx_i &\sim p(\rvx_i \vert   \rvTheta_1) . \label{eq:model_simple}
    \end{align}
\end{repobservation}

\begin{proof}
    Substituting $p(\rvtheta)$ into Equation \ref{eq:CSFP} yields

    \begin{align}
      \rvu^{*}=  \argmin_{\rvu \in \gU}\E_{\rvTheta \sim \Q^{0, \delta_0}}\left[\frac{1}{2\gamma}\int_0^1\left\|\rvu(t, \rvTheta_t)\right\|^2 dt - \ln p(\mX \vert   \rvTheta_1)\right] \label{eq:CSFP_2}.
    \end{align}

    Then, from \citep{boue1998variational,tzen2019neural,tzen2019theoretical} we 
    know that the term $\E\left[\int_0^1\left\|\vu_t\right\|^2 dt - \ln p(\mX \vert   \rvTheta_1)\right]$ is the negative ELBO of the model specified in Equation \ref{eq:model_simple}. 
\end{proof}

While the above observation highlights a specific connection between N-SFP and 
traditional VBI (Variational Bayesian Inference), it is limited to Bayesian models 
that are specified with Gaussian priors. To extend the result, we take inspiration 
from the recursive nature of Bayesian updates in the following result.

% \begin{replemma}{lem:vi}
%     The SBP

%     \begin{equation}
%             \inf_{\Q \in \calD\left(\delta_0,\; p(\rvtheta\vert   \mX)\right)} \KL\left(\Q \big\vert \big\vert \sS\right), \label{eq:sbp_delta}
%     \end{equation}

%     with reference process $\sS$ described by

%     \begin{align}
%         d\rvTheta_t\!&=\!\nabla \ln Q^\gamma_{1-t}\left[\frac{p(\rvTheta_t)}{\gN(\rvTheta_t\vert  \bm{0}, \gamma \I)}\right] +\! \sqrt{\gamma}  d\rvw_t, 
%         \quad \rvTheta_0 \sim \delta_{0} , %\label{eq:prior_foll}
%     \end{align}

%     corresponds to maximising the ELBO of the model:

%     \begin{align}
%         d\rvTheta_t\!&=\!\nabla \ln Q^\gamma_{1-t}\left[\frac{p(\rvTheta_t)}{\gN(\rvTheta_t\vert  \bm{0}, \gamma \I)}\right] +\! \sqrt{\gamma}  d\rvw_t, \quad \rvTheta_0 \sim \delta_{0},\nonumber \\
%         \rvx_i &\sim p(\rvx_i \vert   \rvTheta_1) . %\label{eq:model_simple_2}
%     \end{align}

% \end{replemma}

\begin{lemma}\label{lem:vi}
    The SBP $\;\inf_{\Q \in \calD\left(\delta_0,\; p(\rvtheta\vert   \mX)\right)} \KL\left(\Q \big\vert  \big\vert   \sS\right)$
    with reference process $\sS$ described by
    \begin{align}
        \rvTheta_0 &\sim \delta_{0} \\
        d\rvTheta_t\!=\!\nabla \ln Q^\gamma_{1-t}&\left[\frac{p(\rvTheta_t)}{\gN(\rvTheta_t\vert  \bm{0}, \gamma \I)}\right] +\! \sqrt{\gamma}  d\rvw_t,  
         \label{eq:prior_foll}
    \end{align}
    corresponds to maximising the ELBO of the model:
    \begin{align}
    \rvTheta_0 &\sim \delta_{0}\nonumber, \\
      \;\; d\rvTheta_t\!=\!\nabla \ln Q^\gamma_{1-t}&\left[\frac{p(\rvTheta_t)}{\gN(\rvTheta_t\vert  \bm{0}, \gamma \I)}\right] +\! \sqrt{\gamma}  d\rvw_t,  ,\nonumber\\
        \rvx_i &\sim p(\rvx_i \vert   \rvTheta_1) 
        . \label{eq:model_simple_2}
    \end{align}
\end{lemma}

\begin{proof}
    For brevity let $\rvu^0(t, \vtheta)=\nabla \ln Q^\gamma_{1-t}\left[\frac{p(\vtheta)}{\gN(\vtheta\vert  \bm{0}, \gamma \I)}\right] $. 
    First notice that the time-one marginals of $\sS$ are given by the 
    Bayesian prior:

    \begin{align*}
        (\rvTheta_1)_{\#} \sS = p(\vtheta) d\vtheta
    \end{align*}

    Now from \cite{leonard2012schrodinger,pavon2018data} we know that the Schrödinger 
    system is given by:

    \begin{align}
        \label{eq:SchrSys1}
        \phi_0(\rvtheta_0) \int p(\vtheta_0,0, \vtheta_1,1) \hat{\phi}_1(\vtheta_1)d\vtheta_1 &= \delta_0(\vtheta_0), \\
        \hat{\phi}_1(\rvtheta_1) \int p(\vtheta_0,0,\vtheta_1, 1){\phi}_0(\vtheta_0)d\vtheta_0 &= p(\vtheta_1 \vert   \mX),
    \end{align}

    where Equation \ref{eq:SchrSys1} can be given a rigorous meaning in weak form (that is, by integrating against suitable test functions).
    Notice $ \phi_0=\delta_0$ and thus it follows that

    \begin{align}
        \hat{\phi}_1(\vtheta)  = \frac{p(\vtheta \vert   \mX)}{p(0,0,\vtheta, 1)} = \frac{p(\vtheta \vert   \mX)}{p(\vtheta)}=\frac{p(\mX \vert   \vtheta)}{\gZ}.
    \end{align}

    By \cite{pavon1989stochastic,dai1991stochastic,pavon2018data} the optimal drift is given by:

    \begin{align}
        \rvu^*(t, \vtheta) = \gamma \nabla \ln \E[p(\mX \vert \rvTheta_1) \vert \rvTheta_t = \vtheta], 
    \end{align}

    where the expectation is taken with respect to the reference process $\sS$. 
    Now if we let $v(\vtheta,t) =-\ln \E[p(\mX \vert   \rvTheta_1) \vert   \rvTheta_t = \vtheta]$ 
    be our value function then via the linearisation of the Hamilton-Bellman-Jacobi 
    Equation through Fleming's logarithmic transform \citep{kappen2005linear,thijssen2015path,tzen2019theoretical}
    it follows that said value function satisfies:

    \begin{align}
        v(\vtheta, t) =  \min_{\rvu \in \gU}\E\left[\frac{1}{2\gamma}\int_t^1\left\|\rvu(t, \rvTheta_t) - \rvu^0(t, \rvTheta_t)\right\|^2 dt - \ln p(\mX \vert   \rvTheta_1)\Big\vert   \rvTheta_t = \vtheta\right],
    \end{align}

    and thus $\rvu^*(t, \vtheta) = \gamma \nabla \ln \E[p(\mX \vert   \rvTheta_1) \vert   \rvTheta_t = \vtheta]$ is a minimiser to:

    \begin{align}
        \rvu^{*}=  \argmin_{\rvu \in \gU}\E\left[\frac{1}{2\gamma}\int_0^1\left\|\rvu(t, \rvTheta_t) -\rvu^0(t, \rvTheta_t)\right\|^2 dt - \ln p(\mX \vert   \rvTheta_1)\right] .
    \end{align}
\end{proof}

\section{Stochastic Variational Inference} \label{apdx:svi}

For a Bayesian model having the structure specified by~\eqref{eq:bayes_inf} the 
objective in~\eqref{eq:CSFP} can be written as follows:

\begin{align}
    \E_{\rvTheta \sim \Q^{\rvu, \delta_0}}&\left[\frac{1}{2\gamma}\int_0^1\left\|\rvu(t,\rvTheta_t)\right\|^2 dt - \ln \frac{p(\mX \vert \rvTheta_1)p(\rvTheta_1)}{\gN(\rvTheta_1 \vert \bm{0}, \gamma \I)}\right]  \nonumber \\
                                                                                                                     &= \E\left[\frac{1}{2\gamma}\int_0^1\left\|\rvu(t,\rvTheta_t)\right\|^2 dt - \ln\frac{p(\rvTheta_1)}{\gN(\rvTheta_1 \vert \bm{0}, \gamma \I)}\right] + \sum_{i=1}^N \E\left[\ln p(\rvx_i \vert \rvTheta_1)\right],
\end{align}

where the last term can be written as:

\begin{equation}
    \sum_{i=1}^N \E\left[\ln p(\rvx_i \vert   \rvTheta_1)\right] = \frac{N}{B} \E_{\rvx_{k_i} \sim \calD}\left[\sum_{i=1}^B\E\left[\ln p(\rvx_{k_i} \vert   \rvTheta_1)\right]\right]
\end{equation}

\noindent
That is, it is possible to obtain an unbiased estimate of the objective (and its gradients) by subsampling the data with random batches of size $B$ and using the scaling $\frac{N}{B}$. A version of the algorithm with Euler-Maruyama discretization of the SDE is given in Algorithm~\ref{alg:svi}.

\section{Decoupled Drift Results} \label{apdx:decoup}

First let us consider the setting where the local variables are fully independent, 
that is, $\vtheta_i\bigCI\vtheta_j$.

\begin{remark}
    The heat semigroup preserves fully factored (mean-field) distributions thus the Föllmer drift is decoupled.
\end{remark}

In this setting we can parametrise the dimensions of the drift which correspond 
to local variables in a decoupled manner, 
$[\rvu_{t}]_{\vtheta_i} =u^{\vtheta_i}(t, \vtheta_i, \vx_i)$.
This amortised parametrisation \citep{kingma2013auto} allows us to carry out 
gradient estimates using a mini-batch \citep{hoffman2013stochastic} rather than 
hold the whole state space in memory.

\begin{repremark}{remark:fail}
    The heat semigroup does not preserve conditional independence structure in the drift. That is, the optimal drift does not decouple and as a result depends on the full state space.
\end{repremark}

\begin{proof}
    Consider the following distribution:

    \begin{align}
        \gN(x\vert z,0)\gN(y\vert  z,0)\gN(z\vert  0,1)  
    \end{align}

    We want to estimate:

    \begin{align}
        \E\left[\frac{\gN(X+x\vert  Z+z,1)  \gN(Y+y\vert  Z+z,1)\gN(Z+z\vert  1,0)}{\gN(X+x\vert  0,1)  \gN(Y+y\vert  0,1)\gN(Z+z\vert  0,1)}  \right],
    \end{align}

    where $X,Y,Z \sim \gN(0, \sqrt{1-t})$. From 

    \begin{align}
        \E\left[\frac{\gN(X+x\vert  Z+z,1)  \gN(Y+y\vert  Z+z,1)}{\gN(X+x\vert  0,1)  \gN(Y+y\vert  0,1) }\right]
    \end{align}

    we can easily see that the above no longer has conditional independence structure and thus when taking its logarithmic derivative the drift does not decouple.
\end{proof}

\begin{repremark}{rem:example}
    An SDE parametrised with a decoupled drift $[\rvu_{t}]_{\vtheta_i} =u(t, \vtheta_i,\Phi,\vx_i)$ can reach transition densities which do not factor.
\end{repremark}

\begin{proof}
    Consider the linear time-homogeneous SDE:

    \begin{align}
        d \rvTheta_t =  \mA \rvTheta_t dt + \gamma d \mW_t, \;\; \rvTheta_0 = 0,
    \end{align}

    where:

    \begin{align}
         [\mA]_{ij} = \delta_{ij} + i\delta_{1j},
    \end{align}

    then this SDE admits a closed form solution:

    \begin{align}
        \rvTheta_t = \gamma \int_0^t \exp\left(\mA (t-s)\right) d\mW_s,
    \end{align}

    which is a Gauss-Markov process with 0 mean and covariance matrix:

    \begin{align}
        \bm{\Sigma}(t)=\gamma^2 \int_0^t \exp\left(\mA (t-s)\right) \exp\left(\mA (t-s)\right)^\top ds
    \end{align}

    We can carry out the matrix exponential through the eigendecomposition of $\mA$, for simplicity let us consider the 3-dimensional case:

    \begin{align}
        \exp\left(\mA (t-s)\right) = S e^{D(t-s)} S^{-1} = \begin{pmatrix}
        0& 1& 1& \\
        1& 0& 2\\
        0& 0& 2
        \end{pmatrix}  \begin{pmatrix}
        e^{t-s}& 0& 0& \\
        0& e^{t-s}& 0\\
        0& 0& e^{3(t-s)}
        \end{pmatrix} \begin{pmatrix}
        0& 1& -1& \\
        1& 0& -1/2\\
        0& 0& 1/2
        \end{pmatrix}
    \end{align}

    From this we see that:

    \begin{align}
        \exp\left(\mA (t-s)\right) \exp\left(\mA (t-s)\right)^{\top} &= S e^{D(t-s)} S^{-1} (S e^{D(t-s)} S^{-1})^{\top}\\
        &= S e^{D(t-s)} S^{-1}  S^{-\top} e^{D(t-s)} S^\top \\
        &= \frac{1}{4} S e^{D(t-s)} \begin{pmatrix}
        8 & 2 & -2 \\
        2 & 5 & -1 \\
        -2 &-1 & 1 \\
        \end{pmatrix}  e^{D(t-s)} S^ \top \\
         &= \frac{1}{4} S \begin{pmatrix}
        8 e^{2(t-s)} & 2e^{2(t-s)} & -2e^{4(t-s)} \\
        2e^{2(t-s)} & 5e^{2(t-s)} & -e^{4(t-s)} \\
        -2e^{4(t-s)} &-e^{4(t-s)} & e^{6(t-s)} \\
        \end{pmatrix}  S^ \top
    \end{align}

    Integrating wrt to $s$ yields:

    \begin{align}
        \int \exp\left(\mA (t-s)\right) \exp\left(\mA (t-s)\right)^{\top} ds &= \frac{1}{4} S \begin{pmatrix}
        4  & 1 & -\frac{1}{2} \\
         1 & \frac{5}{2} & -\frac{1}{4} \\
        -\frac{1}{2} &-\frac{1}{4} & \frac{1}{6} \\
        \end{pmatrix}  S^ \top \\
        = \frac{1}{24} \begin{pmatrix}
        13& 2 & -1 \\
         2 & 16 & -2\\
        -1& -2 & 4 \\
        \end{pmatrix} .
    \end{align}

    The covariance matrix is dense at all times and thus the density $\law (\rvTheta_t) = \gN(\bm{\mu}(t), \bm{\Sigma}(t))$ does not factor (is a fully joint distribution). This example motivates that even with the decoupled drift we can reach coupled distributions.

% Calulations:
% Eigendecomposition: https://www.wolframalpha.com/input/?i=%5B%5B1%2C+0%2C+1%5D%2C+%5B0%2C1%2C2%5D%2C%5B0%2C0%2C3%5D%5D
% S^-1 S^-T = https://www.wolframalpha.com/input/?i=%7B%7B0%2C+1%2C+-1%7D%2C+%7B1%2C+0%2C+-1%2F2%7D%2C+%7B0%2C+0%2C+1%2F2%7D%7D++*%7B%7B0%2C+1%2C+-1%7D%2C+%7B1%2C+0%2C+-1%2F2%7D%2C+%7B0%2C+0%2C+1%2F2%7D%7D++%5ET
%  final product : https://www.wolframalpha.com/input/?i=%7B%7B0%2C+1%2C+1%7D%2C+%7B1%2C+0%2C+2%7D%2C+%7B0%2C+0%2C+2%7D%7D++*+%7B%7B4%2C+1%2C+-1%2F2%7D%2C+%7B1%2C+5%2F2%2C+-1%2F4%7D%2C+%7B-1%2F2%2C+-1%2F4%2C+1%2F6%7D%7D+++*+%7B%7B0%2C+1%2C+1%7D%2C+%7B1%2C+0%2C+2%7D%2C+%7B0%2C+0%2C+2%7D%7D+%5E%28T%29

\end{proof}

\section{Low Variance Estimators and Sticking the Landing} \label{appdx:stl}

\begin{reptheorem}{thrm:stl}
    The STL estimator proposed in \citep{xu2021infinitely} satisfies

    \begin{align}
        \frac{\mathrm{d}}{\mathrm{d}\varepsilon}  \mathcal{F}(\rvu^* + \varepsilon \vphi) \Big\vert_{\varepsilon = 0} = 0,
    \end{align}

    almost surely, for all smooth and bounded perturbations $\vphi$.
\end{reptheorem}

\begin{proof}
    Let us decompose $\gF$ in the following way:

    \begin{align}
        \gF(\rvu) = \gF_0(\rvu) + \gF_1(\rvu)
    \end{align}

    where (denoting the terminal cost with $g$):

    \begin{align}
        \gF_0(\rvu) &= \frac{1}{2\gamma}\int_0^1\left\|\rvu(t,\rvTheta_t)\right\|^2 dt + g(\rvTheta_1) \\
        \gF_1(\rvu) &= \frac{1}{\sqrt{\gamma} }\int_0^1\rvu^{\perp}(t,\rvTheta_t)^\top d\rvw_t
    \end{align}

    Denoting $\rvTheta^\rvu \sim \Q^{\rvu, \delta_0}$, from \cite{nusken2021solving}, 
    Theorem 5.3.1, Equation 133 it follows that:

    \begin{align}
        \frac{\mathrm{d}} {\mathrm{d}\varepsilon} \mathcal{F}_0(\rvu^* + \varepsilon \vphi) \Bigg\vert  _{\varepsilon = 0} = -\frac{1}{\sqrt{\gamma} }\int_0^1 \mA_t \cdot(\nabla \rvu_t^*)(\rvTheta_t^{\rvu^*}) \,\mathrm{d}\rvw_t,
    \end{align}

    almost surely, where $\mA_t$ is defined as

    \begin{align}
        \mA^\phi_t = \frac{\mathrm{d}\rvTheta_t^{\rvu^* + \varepsilon \vphi} }{\mathrm{d}\varepsilon} \Bigg\vert  _{\varepsilon = 0}
    \end{align}

    and satisfies:

    \begin{align}
        \mathrm{d}\mA_t^\phi = \phi_t(\rvTheta_t^{\rvu^*}) \, \mathrm{d}t + (\nabla \rvu^*)^\top(\rvTheta_t^{\rvu^*}) \mA_t^\phi \, \mathrm{d}t, \qquad \mA_0^\phi = 0.
    \end{align}

    Similarly via the chain rule it follows that:

    \begin{align}
        \frac{\mathrm{d}} {\mathrm{d}\varepsilon} \mathcal{F}_1(\rvu^* + \varepsilon \vphi) \Bigg\vert  _{\varepsilon = 0} =\frac{\mathrm{d}}{\mathrm{d}\varepsilon} \left( \frac{1}{\sqrt{\gamma} }\int_0^1 \rvu_t^*(\rvTheta_t^{\rvu^* + \varepsilon \vphi})^\top \mathrm{d}\rvw_t \right) \Bigg\vert  _{\varepsilon = 0} = \frac{1}{\sqrt{\gamma} }\int_0^1 \mA_t^\phi \cdot (\nabla \rvu_t^*)(\rvTheta_t^{\rvu^*}) \mathrm{d}\rvw_t
    \end{align}

    almost surely, combining these results we can see that $ \frac{\mathrm{d}}{\mathrm{d}\varepsilon}  \mathcal{F}(\rvu^* + \varepsilon \vphi) \Big\vert  _{\varepsilon = 0} =0$ almost surely as required.
\end{proof}

\section{Stabilising MC-SFS Implementation} \label{apdx:sfs}

We found the estimators proposed in \cite{huang2021schrodinger} (Equations 2.20 or 2.21, and  Algorithm 2 in \cite{huang2021schrodinger}) to be very numerically unstable. Even in two dimensions the montecarlo estimator of the drift evaluated to nans and infs on more than 50\% of the generated samples. This is due to the RND $f$ of Equation \ref{eq:sfs_est} often evaluating to either $0$ due to underflow or a very small number resulting in Equation \ref{eq:sfs_est} becoming very large and unstable.

In order to alleviate this we propose the a novel modified logsmexp reformulation of Equation \ref{eq:sfs_est}:

\begin{lemma} (Stable MC-SFS) \label{lem:stabilizer}
    The MC-SFS estimator

    \begin{align}
         \hat{\rvu}^*(t, \vx) = \frac{\E_{\vz \sim \hat{P}} [\vz_s f(\vx + \sqrt{1-t} \vz)]}{\E_{\vz \sim \hat{P}} {[\sqrt{1-t}} f(\vx + \sqrt{1-t} \vz)]},
    \end{align}

    Where $\hat{P}$ is the empirical measure:

    \begin{align}
        \hat{P} =\frac{1}{S} \sum_{s=1}^S \delta_{\vz_{s}}
    \end{align}

    Can be re-expresssed as:

    \begin{align}
        \hat{\rvu}^*(t, \vx) &= \exp\left(\logsumexp_s  g^+_{\vx}(\vz_s) - \logsumexp_s \ln \gZ_s\right) \\
                             &- \exp\left(\logsumexp_s  g^-_{\vx}(\vz)-   \logsumexp_s \ln \gZ_s\right)
    \end{align}

    where:

    \begin{align}
          g^+_{\vx}(\vz_s) & =   \begin{cases}
    \ln  \vz_s f(\vx + \sqrt{1-t} \vz_s) & \mathrm{if}\; \vz_s > 0\\
       0 & \mathrm{otherwise}
        \end{cases}\\
          g^-_{\vx}(\vz_s)  &=   \begin{cases}
    \ln  \vz_s f(\vx + \sqrt{1-t} \vz_s)    & \mathrm{if}\; \vz_s < 0\\
       0 & \mathrm{otherwise}
        \end{cases}
    \end{align}

    and $\ln \gZ_s = \ln\sqrt{1-t} + \ln f(\vx + \sqrt{1-t} \vz_s) $

\end{lemma}

\begin{proof}
    Firstly notice that the logsumexp formula cannot be applied to the numerator 
    as the terms $\vz_s f(\vx + \sqrt{1-t} \vz_s)$ in the numerator can take on 
    negative values and thus we cannot take the log. 

    In order to take log the note that $\E_{\hat{P}}[f]$ is a Lebesgue–Stieltjes 
    integral and thus by construction we can decompose it into positive and 
    negative parts:

    \begin{align}
         \hat{\rvu}_t^*(\vx) = \frac{\E_{\vz \sim \hat{P}} [(\vz_s f(\vx + \sqrt{1-t} \vz))]}{\E_{\vz \sim \hat{P}} {[\sqrt{1-t}} f(\vx + \sqrt{1-t} \vz)]} =\frac{\E_{\vz \sim \hat{P}} [(\vz_s f(\vx + \sqrt{1-t} \vz))^+]}{\E_{\vz \sim \hat{P}} {[\sqrt{1-t}} f(\vx + \sqrt{1-t} \vz)]}-  \frac{\E_{\vz \sim \hat{P}} [(\vz_s f(\vx + \sqrt{1-t} \vz))^-]}{\E_{\vz \sim \hat{P}} {[\sqrt{1-t}} f(\vx + \sqrt{1-t} \vz)]} 
    \end{align}

    wlog consider the first term:

    \begin{align}
        \frac{\E_{\vz \sim \hat{P}} [(\vz_s f(\vx + \sqrt{1-t} \vz))^+]}{\E_{\vz \sim \hat{P}} {[\sqrt{1-t}} f(\vx + \sqrt{1-t} \vz)]} = \exp\left(\log \E_{\vz \sim \hat{P}} [(\vz_s f(\vx + \sqrt{1-t} \vz))^+] - \ln E_{\vz \sim \hat{P} }[\sqrt{1-t} f(\vx + \sqrt{1-t} \vz)]\right)
    \end{align}

    and similarly for the second, at this point we can trivially apply the log 
    sum exp formula to each of the exponents separately as their integrands 
    are positive.
\end{proof}

For efficient implementation we first separate the samples into positive and negative  and then proceed to compute each of the $g^+$ and $g^-$ terms separately which avoids evaluating any $\log 0$ terms. We found this formula to have no numerical instabilities in our experiments ranging up to high dimensional cases $d=2^{12}$ without issue.

\section{Sensitivity of hyperparameters to Hypespectral Unmixing Results} \label{apdx:sens}

While we were able to find step size schedules for SGLD that would work well
for the Hyperspectral image data, it is important to note that it was due to
heavy tuning and a stroke of luck. As shown in~\ref{tab:hype} there are four parameters
to adjust for the step size scheduling of SGLD and the resulting performance is
very sensitive to all of them. To illustrate this, we fixed the parameters associated 
to $\sigma^2$ as given in~\ref{tab:hype}, and varied the others. The resulting
samples are provided in figure~\ref{fig:sgld_sensitivity}.

In contrast, N-SFS has only one tunable parameter, which impacts the results much
less, as shown in figures~\ref{fig:nsfs_sensitivity} and~\ref{fig:nsfs_decoupled_sensitivity}.

\begin{figure}
    \centering
    \includegraphics[width=\textwidth]{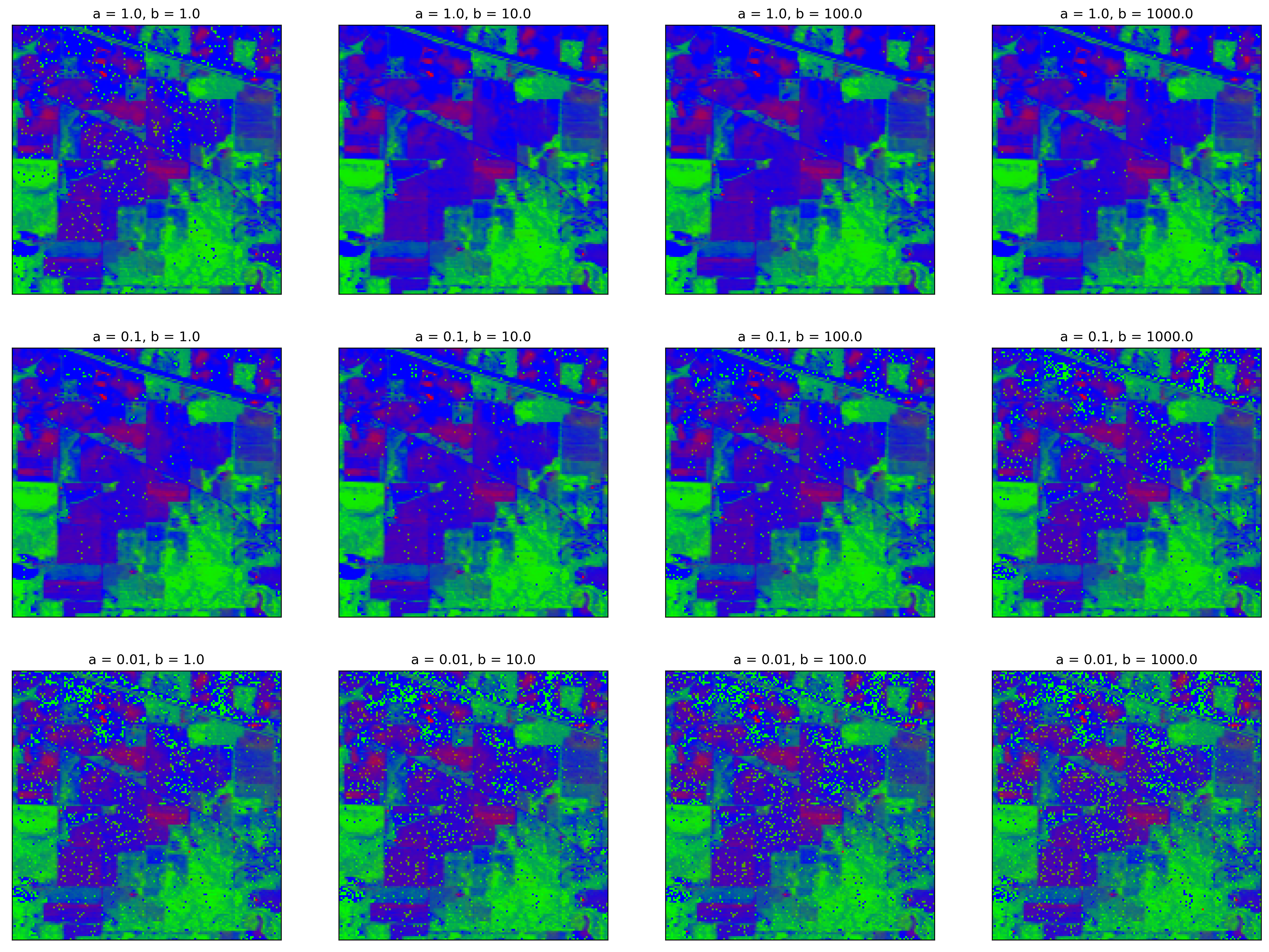}
    \caption{SGLD sensitivity to step size scheduling}\label{fig:sgld_sensitivity}
\end{figure}

\begin{figure}
    \centering
    \includegraphics[width=\textwidth]{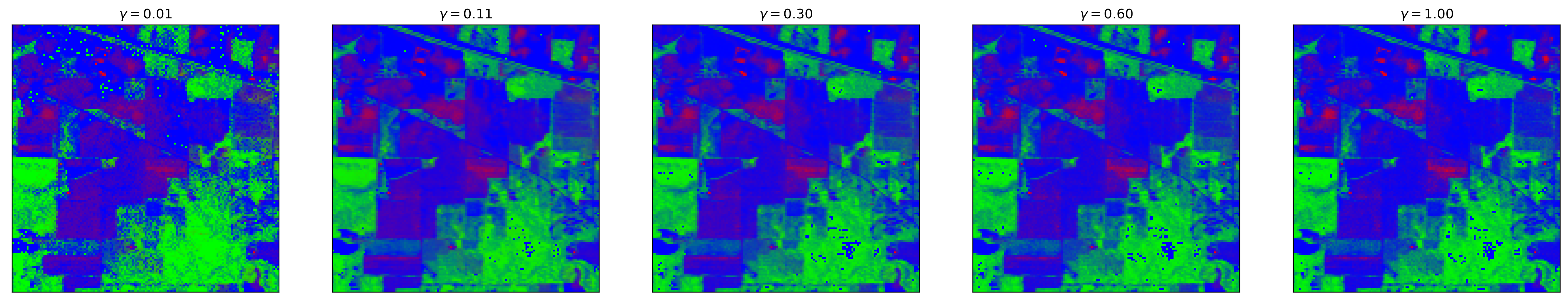}
    \caption{N-SFS sensitivity to $\gamma$}\label{fig:nsfs_sensitivity}
\end{figure}

\begin{figure}
    \centering
    \includegraphics[width=\textwidth]{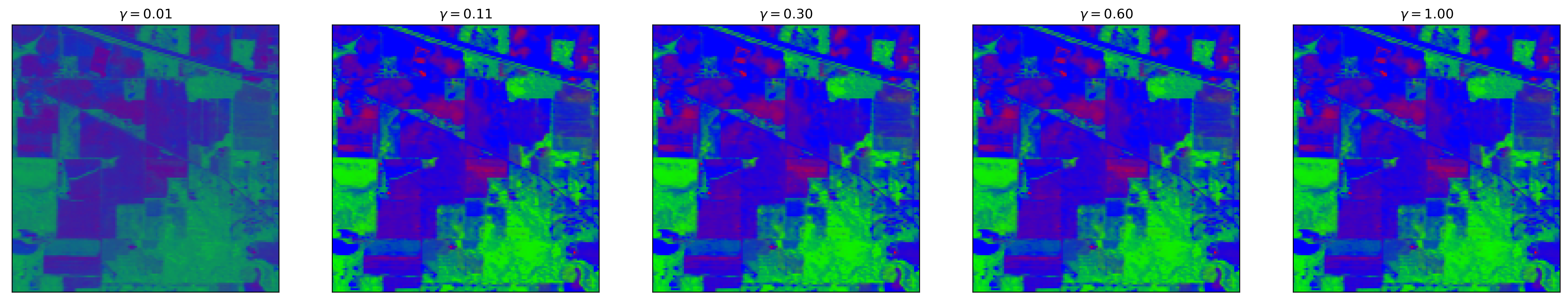}
    \caption{Decoupled N-SFS sensitivity to $\gamma$}\label{fig:nsfs_decoupled_sensitivity}
\end{figure}

\section{Experimental Details and Further Results} \label{apdx:exp}

\subsection{Method Hyperparameters}

In Table \ref{tab:hype} we show the experimental configuration of the trialled 
algorithms across all datasets. For the selected values of $\gamma$ we ran a 
small grid search $\gamma \in \{0.5^2, 0.2^2, 0.1^2, 0.05^2, 0.01^2\}$ and selected 
the $\gamma$ with best training set results.

\subsection{Step Function Dataset}

Here we describe in detail how the step function dataset was generated:

\begin{align}
    y(x) = \mathbbm{1}_{x \geq 0} + \epsilon, \quad \epsilon \sim \gN(0, 0.1)
\end{align}

Where:
\begin{itemize}
    \item  $\sigma_y=0.1$
    \item $N_{\mathrm{train}} = 100$, $N_{\mathrm{test}} = 100$
    \item $x_{\mathrm{train}} \in (-3.5, 3.5)$
    \item $x_{\mathrm{test}} \in (-10, 10)$
\end{itemize}

\begin{landscape}
\begin{table}[]

\caption{Hyper parameter configuration for methods and optimisers.}

\adjustbox{max width=\linewidth}{
    \begin{tabular}{cccccccc}
        \hline
        \multirow{2}{*}{Method} & \multirow{2}{*}{Hyperparameters}      & \multicolumn{6}{l}{Experiments}                                                                                                                                                                                                               \\ \cline{3-8} 
                                &                                       & Step Function                         & MNIST                                 & CIFAR10                               & Hyperspectral Unmixing                & LogReg                                & ICA                                   \\ \hline
        \multirow{9}{*}{N-SFS}  & Optimser                              & Adam                                  & Adam                                  & Adam                                  & Adam / Adam                                  & Adam                                  & Adam                                  \\
                                & Optimiser step size                   & $10^{-4}$                             & $10^{-5}$                             & $10^{-5}$                             & $10^{-5}$                             & $10^{-4}$                             & $10^{-4}$                             \\
                                & $\rvTheta$ batch size                 & $32$                                  & $32$                                  & $32$                                  & 32                                  & $32$                                  & $32$                                  \\
                                & Data batch size                       & $32$                                  & $50$                                  & $50$                                  & Whole dataset                        & Whole train set                       & $10$                                  \\
                                & \# of iterations                      & $300$                                 & $18750$                               & $18750$                               & 2000                                  & $300$                                 & $2832$                                \\
                                & \# of posterior samples               & 100                                   & 100                                   & 100                                   & 20                                      & 100                                   & 100                                   \\
                                & $\gamma$                              & $0.05^2$                              & $0.1^2$                               & $0.05^2$                              & $0.2^2$                                     & $0.2^2$                               & $0.01^2$                              \\
                                & EM train  $\Delta t_{\mathrm{train}}$ & $0.05$                                & $0.05$                                & $0.05$                                & 0.05                                   & $0.05$                                & $0.05$                                \\
                                & EM test  $\Delta t_{\mathrm{test}}$   & $0.01$                                & $0.01$                                & $0.01$                                & 0.01                                      & $0.01$                                & $0.01$                                \\ \cline{2-8} 
        \multirow{7}{*}{SGLD}   & Adaptive step schedule                & $\lambda(i) = \frac{a}{(i+b)^\gamma}$ & $\lambda(i) = \frac{a}{(i+b)^\gamma}$ & $\lambda(i) = \frac{a}{(i+b)^\gamma}$ & $\lambda_{\sigma^2}(i) = \frac{a_{\sigma^2}}{(i+b_{\sigma^2})^\gamma}, \lambda_A(i) = \frac{a_A}{(i+b_A)^\gamma}$ & $\lambda(i) = \frac{a}{(i+b)^\gamma}$ & $\lambda(i) = \frac{a}{(i+b)^\gamma}$ \\
                                & $a$                                   & $10^{-3}$                             & $7\times10^{-5}$                              & $10^{-4}$                             & $a_A = 1.0, a_{\sigma^2} = 10^{-6}$                                     & $10^{-4}$                             & $10^{-4}$                             \\
                                & $b$                                   & $10$                                  & $1$                                   & $1$                                   & $b_A = 10.0, b_{\sigma^2} = 1.0$                                      & $1$                                   & $1$                                   \\
                                & $\gamma$                              & $0.55$                                & $0.55$                                & $0.55$                                & 0.55                                      & $0.55$                                & $0.55$                                \\
                                & Posterior Samples                     & $100$                                 & $100$                                 & $100$                                 & 20                                      & $100$                                 & $100$                                 \\
                                & Data batch size                       & $32$                                  & $32$                                  & $32$                                  & Whole dataset                                     & $32$                                  & $32$                                  \\
                                & \# of iterations                      & $300$                                 & $18750$                               & $18750$                               & 10000                                      & $300$                                 & $2832$                                \\ \cline{2-8} 
        \multirow{3}{*}{SGD}    & step size                             & $10^{-2}$                             & $10^{-1}$                             & $10^{-3}$                             & -                                     & -                                     & -                                     \\
                                & Data batch size                       & $32$                                  & $32$                                  & $32$                                  & -                                     & -                                     & -                                     \\
                                & \# of iterations                      & $300$                                 & $18750$                               & $18750$                               & -                                     & -                                     & -                                     \\ \cline{2-8} 
    \end{tabular}
}

 \label{tab:hype}
\end{table}
\end{landscape}

\subsection{Föllmer Drift Architecture}

Across all experiments (with the exception of the MNIST dataset) we used the same architecture to parametrise the Föllmer drift:

\begin{lstlisting}[language=Python, caption=Simple architecture for drift.] 
class SimpleForwardNetBN(torch.nn.Module):

    def __init__(self, input_dim=1, width=20):
        super(SimpleForwardNetBN, self).__init__()

        self.input_dim = input_dim
        
        self.nn = torch.nn.Sequential(
            torch.nn.Linear(input_dim + 1, width), 
            torch.nn.BatchNorm1d(width, affine=False),
            torch.nn.Softplus(),
            torch.nn.Linear(width, width),
            torch.nn.BatchNorm1d(width, affine=False),
            torch.nn.Softplus(),
            torch.nn.Linear(width, width),
            torch.nn.BatchNorm1d(width, affine=False), 
            torch.nn.Softplus(),
            torch.nn.Linear(width, width),
            torch.nn.BatchNorm1d(width, affine=False),
            torch.nn.Softplus(),
            torch.nn.Linear(width, input_dim)
        )
        
        self.nn[-1].weight.data.fill_(0.0)
        self.nn[-1].bias.data.fill_(0.0)
\end{lstlisting}

Note the weights and biases of the final layer are initialised to $0$ in order to start the process at a Brownian motion matching the SBP prior.

For the MNIST dataset we used the score network proposed in \cite{chen2021likelihood}. We aimed in using this same architecture for the CIFAR10 experiments however we were unable to train it stably.

For Hyperspectral Unmixing dataset we used this architecture for N-SFS with full drift,
but had to devise a different architecture for decoupled drifts, as shown below.

\begin{lstlisting}[language=Python, caption=Score Network architecture for drift.] 
class ResNetScoreNetwork(torch.nn.Module):

    def __init__(self, input_dim: int, final_zero=False):
        super().__init__()
        res_block_initial_widths = [300, 300, 300]
        res_block_final_widths = [300, 300, 300]
        res_block_inner_layers = [300, 300, 300]

        self.input_dim = input_dim

        self.temb_dim = 128

        # ResBlock Sequence
        res_layers = []
        initial_dim = input_dim
        for initial, final in zip(res_block_initial_widths, res_block_final_widths):
            res_layers.append(ResBlock(initial_dim, initial, final, res_block_inner_layers, torch.nn.Softplus()))
            initial_dim = initial + final
        self.res_sequence = torch.nn.Sequential(*res_layers)

        # Time FCBlock
        self.time_block = torch.nn.Sequential(torch.nn.Linear(self.temb_dim, self.temb_dim * 2), torch.nn.Softplus())

        # Final_block
        self.final_block = torch.nn.Sequential(torch.nn.Linear(self.temb_dim * 2 + initial_dim, input_dim))
\end{lstlisting}

\begin{lstlisting}[language=Python, caption=Decoupled Drift network for local parameters]
class DecoupledDrift(AbstractDrift):

    def __init__(self, global_dim=1, local_dim=1, data_dim=1, width=20):
        super(DecoupledDrift, self).__init__()

        self.global_dim = global_dim
        self.local_dim = local_dim
        self.data_dim = data_dim

        self.nn = torch.nn.Sequential(
            torch.nn.Linear(global_dim + local_dim + data_dim + 1, width), torch.nn.BatchNorm1d(width, affine=False), torch.nn.Softplus(),
            torch.nn.Linear(width, width), torch.nn.BatchNorm1d(width, affine=False), torch.nn.Softplus(),
            torch.nn.Linear(width, width), torch.nn.BatchNorm1d(width, affine=False), torch.nn.Softplus(),
            torch.nn.Linear(width, width), torch.nn.BatchNorm1d(width, affine=False), torch.nn.Softplus(),
            torch.nn.Linear(width, local_dim)
        )

        self.nn[-1].weight.data.fill_(0.0)
        self.nn[-1].bias.data.fill_(0.0)

\end{lstlisting}

\subsection{BNN Architectures} \label{apdx:bnn}

For the step function dataset we used the following architecture:

\begin{lstlisting}[language=Python, caption=Architecture for step function dataset.] 
class DNN_StepFunction(torch.nn.Module):
    
    def __init__(self, input_dim=1, output_dim=1):
        super(DNN, self).__init__()
        
        self.output_dim = output_dim
        self.input_dim = input_dim
        
        self.nn = torch.nn.Sequential(
            torch.nn.Linear(input_dim, 100),
            torch.nn.ReLU(),
            torch.nn.Linear(100, 100),
            torch.nn.ReLU(),
            torch.nn.Linear(100, output_dim)
        )
\end{lstlisting}

For LeNet5 the architecture used was:

\begin{lstlisting}[language=Python, caption=Architecture for MNIST.] 
class LeNet5(torch.nn.Module):

    def __init__(self, n_classes):
        super(LeNet5, self).__init__()
        
        self.feature_extractor = torch.nn.Sequential(
            torch.nn.Conv2d(
                in_channels=1, out_channels=6,
                kernel_size=5, stride=1
            ),
            torch.nn.Tanh(),
            torch.nn.AvgPool2d(kernel_size=2),
            torch.nn.Conv2d(
                in_channels=6, out_channels=16,
                kernel_size=5, stride=1
            ),
            torch.nn.Tanh(),
            torch.nn.AvgPool2d(kernel_size=2),
        )

        self.classifier = torch.nn.Sequential(
            torch.nn.Linear(in_features=256, out_features=120),
            torch.nn.Tanh(),
            torch.nn.Linear(in_features=120, out_features=84),
            torch.nn.Tanh(),
            torch.nn.Linear(in_features=84, out_features=n_classes),
        )

\end{lstlisting}
The same layer structure as in LeNet5 was used for the CIFAR10 dataset,and  with a difference in the number of channels and size of filters. Exact details can be found in the code repository.

\subsection{Likelihood and Prior Hyperparameters}

In Table \ref{tab:hypmod} we describe the hyperparameters of each Bayesian model as well as their priors and likelihood.

\begin{table}[h!]
\centering
\begin{tabular}{@{}lll@{}}
\toprule
\multirow{2}{*}{Model}         & \multirow{2}{*}{Hyperparameters} & \multirow{2}{*}{Values}                                         \\
                               &                                  &                                                                 \\ \midrule
\multirow{4}{*}{Step Function} & Prior                            & $\gN(\bm{0}, \sigma_\theta^2\mathbb{I})$                        \\
                               & Likelihood                       & $\gN(\vy_i\vert  f_\vtheta(\vx_i), \sigma_y^2\mathbb{I})$             \\
                               & $\sigma_\theta$                  & $1$                                                             \\
                               & $\sigma_y$                       & $0.1$                                                           \\ \cmidrule(l){2-3} 
\multirow{3}{*}{MNIST}         & Prior                            & $\gN(\bm{0}, \sigma_\theta^2\mathbb{I})$                        \\
                               & Likelihood                        & $\mathrm{Cat}(f_\theta(\vx_i))$                               \\
                               & $\sigma_\theta$                  & $1$                                                             \\ \cmidrule(l){2-3}
\multirow{3}{*}{CIFAR10}         & Prior                            & $\gN(\bm{0}, \sigma_\theta^2\mathbb{I})$                        \\
                               & Likelihood                        & $\mathrm{Cat}(f_\theta(\vx_i))$                               \\
                               & $\sigma_\theta$                  & $1$                                                             \\ \cmidrule(l){2-3}
\multirow{3}{*}{Hyperspectral Unmixing}         & Prior                            & $p(\sigma^2) = \ind_{[0, 1]}(\sigma^2), p(a_p) = \ind_{\psimplex{R}}(\ba_p)$                        \\
                               & Likelihood                        & $\gN(\bM\ba_p; \vert  \vert  \ba_p\vert  \vert  ^2\sigma^2I)$                               \\ \cmidrule(l){2-3} 
\multirow{3}{*}{Log Reg}       & Prior                            & $\mathrm{Laplace}(\bm{0},\sigma_\theta,)$                       \\
                               & Likelihood                       & $\mathrm{Bern}(\mathrm{Sigmoid}_{\theta})$                       \\
                               & $\sigma_\theta$                  & $1$                                                             \\ \cmidrule(l){2-3} 
\multirow{3}{*}{ICA}           & Prior                            & $\gN(\bm{0}, \sigma_\theta^2\mathbb{I})$                        \\
                               & Likelihood                       & $\prod_i \frac{1}{4\cosh^2(\frac{\vtheta_i^{\top}\bm{x} }{2})}$ \\
                               & $\sigma_\theta$                  & 1                                                               \\ \bottomrule
\end{tabular} \label{tab:hypmod}
\caption{Specification of Bayesian models.}
\end{table}

\end{appendices}

% TEST
\bibliography{sn-bibliography}

% \bibliography{l4dc2022}



\end{document}